\documentclass[journal]{IEEEtran}

\usepackage{amssymb,amsmath,graphicx,paralist,subfigure,pdfpages,cite}
\usepackage{amsmath,amsthm,amssymb,paralist,subfigure,cite,graphicx,algorithm,url,tikz,wasysym,,pdfpages,algorithmic,algorithm}
\newtheorem{lem}{Lemma} 
\newtheorem{thm}{Theorem}
\newtheorem{definition}{Definition}

\def\ln{{\rm ln}}

\def\mc{\mathcal}
\def\mb{\mathbf}
\def\mbb{\mathbb}
\def\ra{\rightarrow}
\IEEEoverridecommandlockouts

\begin{document}
\title{\textcolor{black}{An opportunistic linear-convex algorithm for localization in mobile robot
networks}}
\author{Sam Safavi,~\emph{Student Member,~IEEE}, and Usman A. Khan$^\dagger$,~\emph{Senior Member,~IEEE}\thanks{$^\dagger$The authors are with the Department of Electrical and Computer Engineering, Tufts University, 161 College Ave, Medford, MA 02155, {\texttt{\{sam.safavi@,khan@ece.\}tufts.edu}}. This work is partially supported by an NSF Career award: CCF \# 1350264.}}
\maketitle
\thispagestyle{empty}

\begin{abstract}
In this paper, we develop a \textcolor{black}{\emph{distributed}} algorithm to localize a network of robots moving arbitrarily in a bounded region. In the case of such mobile networks, the main challenge is that the robots may not be able to find nearby robots to implement a distributed algorithm. We address this issue by providing an opportunistic algorithm that only implements a location update when there are nearby robots and does not update otherwise. We assume that each robot measures a noisy version of its motion and the distances to the nearby robots. 
To localize a network of mobile robots in~$\mathbb{R}^m$, we provide a simple \emph{linear} update, which is based on barycentric coordinates and is linear-convex.
We abstract the corresponding localization algorithm as a Linear Time-Varying (LTV) system and show that it asymptotically converges to the true locations~of~the robots.

We first focus on the noiseless case, where the distance and motion vectors are known (measured) perfectly, and provide sufficient conditions on the convergence of the algorithm. We then evaluate the performance of the algorithm in the presence of noise and provide modifications to counter the undesirable effects of noise. \textcolor{black}{We further show that our algorithm precisely tracks a mobile network as long as there is at least one known beacon (a node whose location is perfectly known).}
\end{abstract}

\section{Introduction}\label{sec1}
In order to navigate reliably and perform useful tasks in robotic networks, a mobile robot must know its exact location. Robot localization, estimation of a robot's location from sensor data, is thus a fundamental problem in providing autonomous capabilities to a mobile robot.
Although some robotic systems rely on a Global Positioning System (GPS) to determine their location in a global reference frame, it is impractical to use GPS in many indoor applications. Localization problems in  robotics can be broadly grouped into two categories: position
tracking; and global localization.
Position tracking, e.g.,~\cite{burgard1997fast}, also known as local localization, requires the knowledge of robots' starting locations, whereas in global localization, e.g.,~\cite{fox1999markov}, no prior estimate of the initial locations exists. 
\textcolor{black}{In terms of the information used for estimating the locations, localization schemes can be classified as range-based,\cite{dil2006range,zhang2005range},
and range-free,\cite{hu2004localization,rudafshani2007localization}. 
While the former depends on measuring the distance and/or angle between the nodes, 
the latter makes no assumptions about the availability of such information and relies on the connectivity of the network.}

Since the early work on navigation with autonomous mobile robots, a variety of  centralized and distributed  techniques have been proposed to tackle the localization problem. Despite higher accuracy in small-sized networks, centralized localization schemes suffer from scalability issues, and are not feasible in large-scale networks. In addition, comparing to the distributed algorithms, the centralized schemes are less reliable and require higher computational complexity due to, e.g., the accumulated inaccuracies caused by multi-hop communications over a wireless network. 

Robot localization approaches include but are not limited to dead-reckoning,~\cite{borenstein1996measurement,cho2011dead,fu2013precise},
Simultaneous Localization and Mapping (SLAM),\textcolor{black}{\cite{cummins2008fab,frese2005multilevel,bailey2006simultaneous,grisetti2007improved,engel2014lsd}},  
Monte Carlo techniques~\cite{hu2004localization,rudafshani2007localization,dil2006range,dellaert1999monte,thrun2001robust}, and
\textcolor{black}{Kalman Filtering  methods~\cite{A1roumeliotis2002distributed,A2carrillo2013decentralized,A3li2013cooperative,A4wanasinghe2014decentralized,A5martinelli2005multi,martinelli2005observability,ahmad2013extended,biswas2013multi,d2015mobile}}; 
other related works include\textcolor{black}{\cite{7862901,thomas2005revisiting,zhou2009efficient,maxim2008trilateration,betke1997mobile,mourikis2006performance,mourikis2006optimal,dieudonne2010deterministic,zhou2011determining,franchi2013mutual,6469211,6631396,mirkhani2013novel}}. We briefly describe the related work below.

\textcolor{black}{Dead-reckoning,~\cite{borenstein1996measurement,cho2011dead}, is a common method to estimate the location of a mobile robot. It uses the wheel rotation measurements to compute the offset from a known starting position. Despite the low cost, simplicity, and easy implementation in real time, dead reckoning methods are prone to accuracy problems due to accumulating wheel slippage errors, which grow without bound over time. Therefore, these methods are only suitable for applications where the robots have good estimates of their initial locations, and their tasks involve exploring only short distances.}

\textcolor{black}{When the map of the environment is not available a priori, Simultaneous Localization and Mapping techniques,\textcolor{black}{\cite{cummins2008fab,frese2005multilevel,bailey2006simultaneous,grisetti2007improved}}, can be used to build a map of an unexplored environment
by a mobile robot, while simultaneously navigating the environment using the map. The main disadvantage of most SLAM algorithms is the high computational complexity, which makes them less efficient specially in larger multi-robot networks. \textcolor{black}{Another disadvantage of SLAM-based solutions is the sharing of the map when multiple robots are involved.}}

\textcolor{black}{When ranging data is noisy, estimation-based localization techniques are widely used.  Sequential Bayesian Estimation (SBE) methods use the recursive Bayes rule to estimate the likelihood of a robot's location. The solution to SBE is generally intractable and cannot be determined analytically. An alternative approach is Kalman-based techniques, which are only optimal when the uncertainties are Gaussian and the system dynamics are linear. However, localization has always been considered as a nonlinear problem and hence the optimality of Kalman-based solutions are not guaranteed.} 

\textcolor{black}{To address the nonlinear nature of localization problems, other suboptimal solutions to approximate the optimal Bayesian estimation include Particle Filters (PF) and Extended Kalman Filters (EKF).} In particular, Sequential Monte Carlo (SMC) method is a PF that exploits posterior probability to determine the future location of a robot. Monte Carlo Localization (MCL) algorithms can solve global localization in a robust and efficient way.
\textcolor{black}{
For example, Ref.~\cite{hu2004localization} introduces the Monte Carlo Localization (MCL) method, which exploits mobility to improve the accuracy of localization. Inspired by~\cite{hu2004localization}, the authors in~\cite{rudafshani2007localization} propose Mobile and Static sensor network Localization (MSL*) that extends MCL to the case where some or all nodes are static or mobile. On the other hand, Ref~\cite{dil2006range} provides Range-based Sequential Monte Carlo Localization method (Range-based SMCL), which combines range-free and range-based information to improve localization performance in mobile sensor networks.}
However, MCL methods are time-consuming as they need to keep sampling and filtering until enough samples are obtained to represent the posterior distribution of a mobile robot's position\textcolor{black}{\cite{4557738,thrun2001robust}.}

\textcolor{black}{On the other hand, Extended Kalman Filter (EKF) approaches,~\cite{A1roumeliotis2002distributed,A2carrillo2013decentralized,A3li2013cooperative,A4wanasinghe2014decentralized,A5martinelli2005multi,martinelli2005observability}, provide suboptimal solutions by linearizing the measurements around the robot's current position estimate,~\cite{A5martinelli2005multi}.} 
\textcolor{black}{In particular, Ref.~\cite{A1roumeliotis2002distributed}} provides a distributed EKF algorithm to estimate the location as well as the rotation of a robot by decomposing a single KF into a number of smaller communication filters. Improvements based on \textit{Covariance Intersection} (CI) is provided in~\cite{A2carrillo2013decentralized,A3li2013cooperative,A4wanasinghe2014decentralized}. \textcolor{black}{The problem of relative localization of two mobile robots in $\mathbb{R}^2$ is considered in~\cite{martinelli2005observability}. By applying the observability rank
	condition for nonlinear systems, the authors in ~\cite{martinelli2005observability} show that the underlying physical nonlinear system of Cooperative
	Localization (CL) in general has three unobservable degrees of freedom, corresponding to the global position
	and orientation, i.e., the system is observable only for more than three landmarks. When both robots move simultaneously, the error grows unbounded and the
	configuration is not observable.}

\textcolor{black}{Relevant to our work, are~\cite{zhou2011determining} and \cite{franchi2013mutual}; Ref.~\cite{zhou2011determining} addresses CL in 3D using combinations of
	range and bearing measurements in conjunction with ego-motion
	measurements. 
	Ref.~\cite{franchi2013mutual} on the other hand, proposes a decentralized method to perform mutual localization in multi-robot systems using relative measurements that do not include the identity of the measured robot.}

In this paper, we provide a \textit{distributed} algorithm to solve \textit{global} localization in a network of mobile robots, i.e., we assume that there is no prior estimates of the initial locations.
\textcolor{black}{As a motivating scenario, consider a multi-robot network of ground/aerial vehicles with no central or local
	coordinator and with limited communications, whose task is to
	transport goods in an indoor facility, where GPS signals are
	not available. In order to perform a delivery task, each mobile
	robot has to know its own location first. In such settings, we
	are interested in developing distributed algorithms to track the
	robot locations such that the convergence is invariant to the
	initial position estimates. However, the implementation of distributed
	algorithms in mobile networks, \textcolor{black}{where both robots and known beacon(s) are moving}, is not straightforward due to the following challenges:}
\begin{inparaenum}[(i)]
\item no robot may be in proximity of any \textcolor{black}{known beacon}
\textcolor{black}{(hereinafter, referred to as a beacon for simplicity)};
\item
a robot may not be able to find nearby robots at all times to perform a distributed algorithm;
and, \item the dynamic neighborhood at each robot results into a time-varying distributed algorithm, 
   whose stability (convergence) analysis is non-trivial.
\end{inparaenum}

\textcolor{black}{We address these challenges by providing an opportunistic update scenario, where a robot updates its location estimate in $\mathbb{R}^m$ only if it lies inside a convex hull of $m+1$ neighbors. Such neighbors are referred to as a \textit{triangulation set}. Using this approach, we show that the robot location estimates are improved as the procedure continues and the algorithm is \textit{optimal}, i.e., it tracks the true robot
locations.}
\textcolor{black}{In this context, the main contribution of this work is to develop a \textit{linear} framework for localization that enables us to circumvent the challenges posed by the predominant  nonlinear approaches to this problem. This linear framework is not to be interpreted as a linearization of an existing nonlinear algorithm. Instead, the nonlinearity from range to location is embedded in an alternate representation provided by the barycentric coordinates.}

We abstract our localization algorithm as an LTV system, whose system matrices may be
\emph{stochastic} or \emph{sub-stochastic}.
{\textcolor{black}{These system matrices do not belong to a finite or a countable set rendering many of the existing results inapplicable. In addition, since they can be either stochastic or sub-stochastic, their product is no longer a group under multiplication}}. Thus, establishing the asymptotic behavior of the underlying LTV system is non-trivial. \textcolor{black}{To address these challenges, we apply a novel method,~\cite{DBLP:journals/corr/SafaviK14,7526779}, to study LTV convergence by partitioning the entire set of
	system matrices into slices and relating the convergence rate
	to the slice lengths. In particular, we show that the algorithm
	converges if the slice lengths do not grow faster than a certain
	exponential rate.}
Next, since our
localization scheme is based on the motion and the distance
measurements, it is meaningful to evaluate the performance of
the algorithm when these parameters are corrupted by noise.
Hence, we study the impact of noise on the convergence
of the algorithm and provide modifications to counter the
undesirable effects of noise.
Finally, we investigate what role does motion play in the behavior of the localization algorithm, \textcolor{black}{and provide necessary conditions in terms of the number of nodes and the dimensions of motion required by our approach. As a consequence, the proposed method does not work, e.g., in the case of one beacon and \textit{less than three robots} (in $\mathbb{R}^2$). Alternative approaches (such as SLAM and filtering-based localization methods) may be used in such settings, or in scenarios where a robot is isolated from the entire network.} 
%

The rest of this paper is organized as follows. In Section~\ref{sec2}, we formulate the problem.
We propose our localization algorithm 
in Section~\ref{sec3}, followed by the convergence analysis in Section~\ref{sec4}. We investigate the effects of noise in Section~\ref{noise}, while in Section~\ref{sec5} we discuss 
the relationship between the dimension of motion and the minimum number of beacons required. 
\textcolor{black}{We discuss different aspects of the algorithm in Section~\ref{chal}} and present detailed
simulation results in Section~\ref{sec7}. Finally, Section~\ref{sec8} concludes the paper.

\vspace{-1mm}
\section{Preliminaries and problem formulation}\label{sec2}
\textcolor{black}{Consider a network of~$N$ robots with unknown locations in the index set,~$\Omega$, and~$M$ beacons with known locations in the index set,~$\kappa$.} Let~$\Psi=\Omega\cup\kappa$ be the set of all nodes (robots and beacons), possibly mobile, located in~$\mathbb{R}^{m},~m\geq1$. \textcolor{black}{We assume that the robots occupy a \textit{non-trivial configuration}, i.e., all of them do not remain on a low-dimensional subspace in~$\mbb{R}^m$ effectively reducing the localization problem to $\mbb{R}^{m-1}$.} We denote the \emph{true location} of the~$i$-th node,~$i\in\Psi$, at time~$k$, with an~$m$-dimensional row vector,~$\mb{x}_k^{i\ast}\in\mbb{R}^m$, 
where~$k\geq0$ is the discrete-time index. The problem is to find the locations of the mobile robots in the set~$\Omega$, given any initialization of the underlying algorithm. 
\textcolor{black}{In what follows, we explain our system model and 
describe a related convex hull inclusion test}.
We then provide
a set of assumptions, which we use later in the design and analysis of our
 algorithm. 

\vspace{-3mm}
\subsection{\textcolor{black}{System model}}\label{mmm} 
We express the motion as the deviation from the current to the next locations, i.e.,
\begin{eqnarray}\label{Eq1}
\mb{x}_{k+1}^{i\ast} = \mb{x}_{k}^{i\ast} + {\widetilde{\mb{x}}_{k+1}^i},\qquad i\in\Psi,
\end{eqnarray}
in which~${\widetilde{\mb{x}}_{k}^i}$ is the true motion vector at time $k$. We assume that robot~$i$ measures a noisy version,~${\widehat{\mb{x}}_{k}^i}$, of this motion, e.g., \textcolor{black}{by integrating the measurements gathered by an accelerometer}: 
\begin{eqnarray}\label{Eq2}
{\widehat{\mb{x}}_{k}^i}={\widetilde{\mb{x}}_{k}^i}+n^i_{k},
\end{eqnarray}
where $n^i_{k}$ represents the \textcolor{black}{integration noise} at time $k$. We define the true distance between any two nodes, $i$ and~$j$, \textcolor{black}{at the time of communication,~$k$}, as {\textcolor{black}{$\widetilde{d}^{ij}_k$}. We assume that the motion is restricted in a bounded region in~$\mbb{R}^m$ and that the measured distance,~$\widehat{d}^{ij}_k$, at robot $i$ includes noise,~i.e.,
	\begin{equation}\label{Eq3}
	\widehat{d}^{ij}_k = {\textcolor{black}{\widetilde{d}^{ij}_k}} + r^{ij}_k,
	\end{equation}
where $r^{ij}_{k}$ is the {\textcolor{black}{noise in} the distance measurement at time~$k$. 		

\vspace{-2mm}
\subsection{Inclusion test and Barycentric representation}\label{ss:DILOC}
As we will explain, to perform an update, each robot,~$i$, must find~$m+1$ neighbors, say~$\{1,2,3\}$ in~$\mbb{R}^2$, such that it lies \textcolor{black}{strictly in the interior of} their convex hull,~$\mathcal{C}(\cdot)$--a triangle formed by~$\{1,2,3\}$. These neighbors, $\{1,2,3\}$, form what we refer to as a~\emph{triangulation set},~$\Theta_i(k)$, at time $k$. \textcolor{black}{We note that the triangulation set,~$\Theta_i(k)$, has a non-zero volume, i.e.,~$A_{\Theta_i(k)}>0$, where~$A_{\Theta_i(k)}$ represents the~$m$-dimensional volume, area in~$\mathbb{R}^2$ or volume in~$\mathbb{R}^3$, of~$\mathcal{C}(\Theta_i(k))$, see Fig.~\ref{sim1_5}, and can be computed by using the Cayley-Menger (CM) determinant,~\cite{sippl1986cayley}, see Appendix~\ref{cm}}. Clearly,~\textcolor{black}{$i\in\mathcal{C}(\Theta_i(k))$ and~$|\Theta_i(k)|=m+1$}.
Thus to perform an update, a robot implements an inclusion test to verify if it lies inside an available convex hull. In any arbitrary dimension,~$m$, a convex hull inclusion test can be described as follows: 
\begin{eqnarray}\label{convEQ}
i\in\mathcal{C}(\Theta_i(k)),\qquad \mbox{if } \sum_{j\in\Theta_i(k)}A_{\Theta_i(k)\cup\{i\}\setminus j} = A_{\Theta_i(k)},\\
i\notin\mathcal{C}(\Theta_i(k)),\qquad \mbox{if } \sum_{j\in\Theta_i(k)}A_{\Theta_i(k)\cup\{i\}\setminus j} > A_{\Theta_i(k)},
\end{eqnarray}
in which~$``\setminus"$ denotes the set difference. \textcolor{black}{Straightforward
computation of Cayley-Menger determinants in $\mathbb{R}^2$ and $\mathbb{R}^3$, allows us to easily compute and take advantage
of the barycentric coordinates in order to avoid nonlinearity in the localization process. CM determinant provides
a simple equation that extends to any dimensions. It computes the area/volume of m-dimensional simplexes and
reduces to, e.g., Heron’s formula for the area of a triangle in $\mathbb{R}^2$, and volume of a tetrahedron in $\mathbb{R}^3$.}
We note that~$A_{\Theta_i(k)\cup\{i\}\setminus j}$ is the volume (area) of the set~$\Theta_i(k)$ with node~$i$ added and node $j$ removed. Furthermore, the test only requires pairwise distances among the nodes in~$\{i\}\cup\mc{C}_i(\cdot)$, \textcolor{black}{see Section~\ref{chalb} for the associated computation complexity.}
\vspace{-5mm}
\begin{figure}[!h]
	\centering
    \includegraphics[width=70mm]{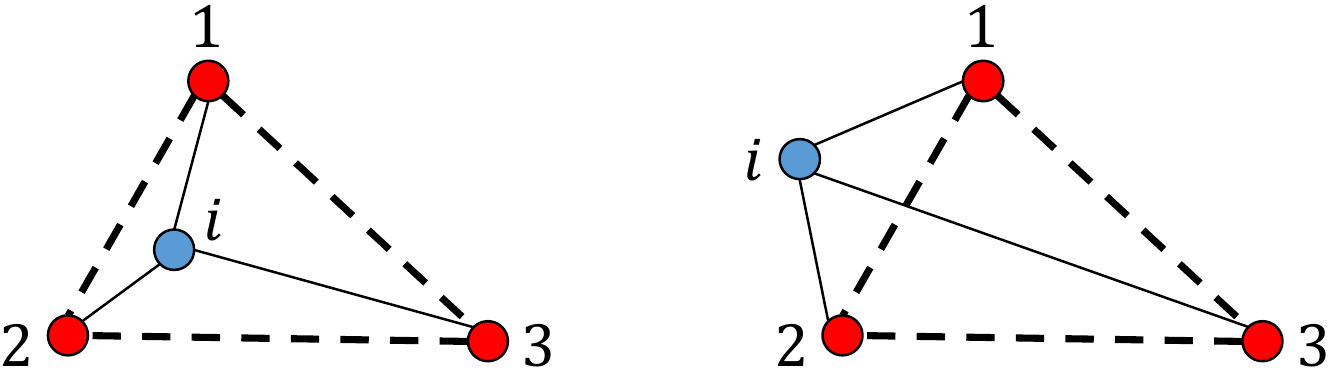}
	\caption{\textcolor{black}{In $\mbb{R}^2$ the inclusion test is passed when the sum of the inner triangles equals the area of the outer triangle (convex hull); (Left) Robot~$i$ lies inside the triangle formed by the neighboring nodes,~$1$,~$2$ and~$3$, and the test is passed; (Right) Inclusion test is not passed.}}
	\label{sim1_5}
\end{figure}

With the help of triangulation sets, we represent the location of node~$i$ at time $k$, ${\mathbf{x}_{k}^{i\ast}}$, as a linear-convex combination of the locations of the neighboring nodes in its triangulation set, see Fig.~\ref{sim1_5} (Left):
\begin{equation}
\textcolor{black}{\mathbf{x}_k^{i\ast}=a_k^{{i}1} \mathbf{x}_k^{1\ast}+a_k^{{i}2} \mathbf{x}_k^{2\ast}+a_k^{{i}3} \mathbf{x}_k^{3\ast}=\sum_{j\in\Theta_i(k)}a_k^{ij}\mb{x}_k^{j\ast},}
\end{equation}
where~\textcolor{black}{$a_k^{ij}$'s} are the \emph{barycentric coordinates}, defined as
\begin{equation}\label{eq7}
\textcolor{black}{a_k^{{i}j}=\frac{A_{\Theta_i(k)\cup\{i\}\setminus j}}{A_{\Theta_i(k)}}}.
\end{equation}
The above representation dates back to the early work by Lagrange and M{\"o}bius,~\cite{mobius1827barycentrische}. \textcolor{black}{Since we assume robot~$i$ to lie strictly inside its triangulation set, all barycentric coordinates are strictly positive, i.e.,~$a_{k}^{ij}>0,\forall j\in\Theta_i(k)$.} 

\subsection{Assumptions}
We now enlist our assumptions:

{\bf A0:} Beacon locations~\textcolor{black}{$\mb{u}_k^{m\ast},m\in\kappa$}, are  known at all times.

{\bf A1:} Each robot, $i\in\Omega$, knows a noisy version,~$\widehat{\mb{x}}_k^i$, of its motion,~$\widetilde{\mb{x}}_k^i$, at all times, see Eq.~\eqref{Eq2}.

{\bf A2:} Each robot, $i\in \Omega$, measures a noisy distance, $\widehat{d}^{ij}_k$, to every node, $j\in\Psi$, in its communication radius,~$r$, at time~$k$, see Eq.~\eqref{Eq3}.

We assume that each robot is able to estimate its distances to the nearby nodes 
by using \textcolor{black}{Received Signal Strength Indicator (RSSI), Time of Arrival (ToA), Time Distance of Arrival (TDoA), Direction of Arrival (DoA)~\cite{hamzehei}}, or camera-based methods,~\cite{PatwariThesiss,camera}. Under the above assumptions, we are interested in finding the true locations of each robot \textcolor{black}{\textit{without the presence of any central or local coordinator.}}
In the following, Sections~\ref{sec3}-\ref{sec4}, we consider the ideal scenario when the motion and distance measurements are not effected by noise, i.e.,~$n_k^i=0$ and $r_k^{ij}=0$, in Eqs.~\eqref{Eq2} and~\eqref{Eq3}. We then study the effects of noise in Section~\ref{noise}, and provide modifications to the algorithm to address the noise on the motion as well as the distance measurements.

\section{Localization Algorithm}\label{sec3}
Consider a network of~$\vert \Omega \vert=N$ mobile robots and~$\vert \kappa \vert=M$ possibly mobile beacons in~$m$-dimensional Euclidean space,~${\mathbb{R}}^m$. We define~${\mathcal{N}}_i(k)\subseteq\Psi$ as the set of neighbors of robot,~$i\in\Omega$, at time~$k$. We now describe our localization algorithm \textcolor{black}{(main steps of the proposed localization algorithm are summarized in Algorithm~\ref{ps})}: At the beginning, each robot starts with a random guess of its location estimate \textcolor{black}{(Alg.~\ref{ps} line \ref{op1})}. At each time $k>0$, we consider the following update scenarios for any arbitrary robot,~$i\in\Omega$: 

{\bf{Case (i)}}: If robot~$i$ does not find at least $m+1$ neighbors,~i.e.,~$0\leq\vert{\mathcal{N}}_i(k)\vert < m+1$, it does not update its current location estimate \textcolor{black}{(Alg.~\ref{ps} lines \ref{op2} and \ref{op3})}.

{\bf{Case (ii)}}: If robot~$i$ finds at least $m+1$ neighbors,~i.e.,~$\vert{\mathcal{N}}_i(k)\vert\geq m+1$, it performs the {inclusion test, described in Section~\ref{ss:DILOC}, on all possible combinations of~$m+1$ neighbors \textcolor{black}{(Alg.~\ref{ps} line \ref{op4})}. If robot~$i$ then fails to find a triangulation set, it does not update \textcolor{black}{(Alg.~\ref{ps} line \ref{op5})}, otherwise it applies the following update\textcolor{black}{\footnote{\textcolor{black}{We consider two update scenarios for case (ii); A robot can update with respect to all possible triangulation sets at time $k$, or alternatively, it can only update once as soon as the inclusion test is passed for the first time at time~$k$.}}} \textcolor{black}{(Alg.~\ref{ps} line \ref{op6})}:
\begin{eqnarray}\label{18}
\mb{x}^i_{k+1} = \alpha_k\mb{x}_k^i + (1-\alpha_k) \sum_{j\in\Theta_i(k)}\textcolor{black}{a_k^{ij}}\mb{x}_k^j + \widetilde{\mb{x}}_{k+1}^i,
\end{eqnarray}
where~$\mb{x}^i_{k}$ is the location estimate of robot $i$ at time $k$,~$\Theta_i(k)$ is the triangulation set at time $k$,~$\textcolor{black}{a_k^{ij}>0}$ is the barycentric coordinates of node~$i$ with respect to the node~$j\in\Theta_i(k)$, and~$\alpha_k$ is such that
\begin{eqnarray}\label{alpk}
\alpha_k=\left\{
\begin{array}{ll}
1, & \forall k~|~\Theta_i(k)=\emptyset,\\
\in\left[\beta,1\right), & \forall k~|~\Theta_i(k)\neq\emptyset,
\end{array}
\right.
\end{eqnarray}
in which $\beta$ is a design parameter, \textcolor{black}{see Section~\ref{4c} (v)}.

{\begin{algorithm}
		\footnotesize
		\caption{\textcolor{black}{Localize $N$ robots in $\mathbb{R}^{m}$ in the presence of $M$ beacons with precision $p$}}
		\label{ps}
		\textcolor{black}{\begin{algorithmic}[1]
				\REQUIRE $M \geq 1~\AND~M + N \geq m+2~\AND~M + \dim \underset{i\in\Omega}\cup \mc{M}_i + \dim \underset {j\in\kappa} \cup \mc{U}_j \geq m+1$.\\
				\STATE $k \leftarrow 0$
				\STATE ${\bf{x}}_0 \leftarrow \mbox{Random initial coordinates}$\label{op1}
				\STATE ${\bf{e}}_0 \leftarrow {\bf{x}}^{\ast}_0-{\bf{x}}_0$
				\WHILE{${\Vert{\bf{e}}_k\Vert}_{2} \geq {10^{-p}}$}
				\STATE $k \leftarrow k+$
				\FOR{$i=1$ to $N$}
				\IF{$0\leq\vert{\mathcal{N}}_i(k)\vert < m+1$}\label{op2}
				\STATE Do not update\label{op3}
				\ELSE
				\STATE Perform the inclusion test on (all possible combinations of) $m+1$ neighbors\label{op4}
				\IF{No triangulation set found}
				\STATE Do not update\label{op5}
				\ELSE
				\STATE Update location estimate according to Eq.~\eqref{18}\label{op6}
				\ENDIF
				\ENDIF
				\ENDFOR
				\STATE ${\bf{e}}_k \leftarrow {\bf{x}}^{\ast}_k-{\bf{x}}_k$
				\ENDWHILE
			\end{algorithmic}}
		\end{algorithm}

A network of $\vert \Omega \vert=5$ mobile robots and $\vert \kappa \vert=2$ beacons with fixed locations in $\mathbb{R}^{2}$ is illustrated in Fig.~\ref{f0}. To clarify the updating scenarios let us just consider robot $1$, shown as a blue filled circle. At time~$k_j$, Fig.~\ref{f0} (a), robot $1$ has only one neighbor
 and does not update, as per case (i).
Although at time $k_l$, Fig.~\ref{f0}~(b), robot $1$ finds $m+1=3$ nodes (two robots and one beacon) within its communication radius, it fails to update because the neighbors do not satisfy the inclusion test. In this case,~$\vert \mc{N}_1(k_l) \vert=3$ and~$\Theta_1(k_l)=\emptyset$. However, robot $1$ updates its location estimates at time $k_m$ and $k_n$, where in the \textcolor{black}{former it finds two different triangulation sets: $\{2,3,4\}$ and $\{2,4,5\}$, and disregards $\{3,4,5\}$ and $\{2,3,5\}$.}
\begin{figure}[!h]
	\centering
	\includegraphics[width=60mm]{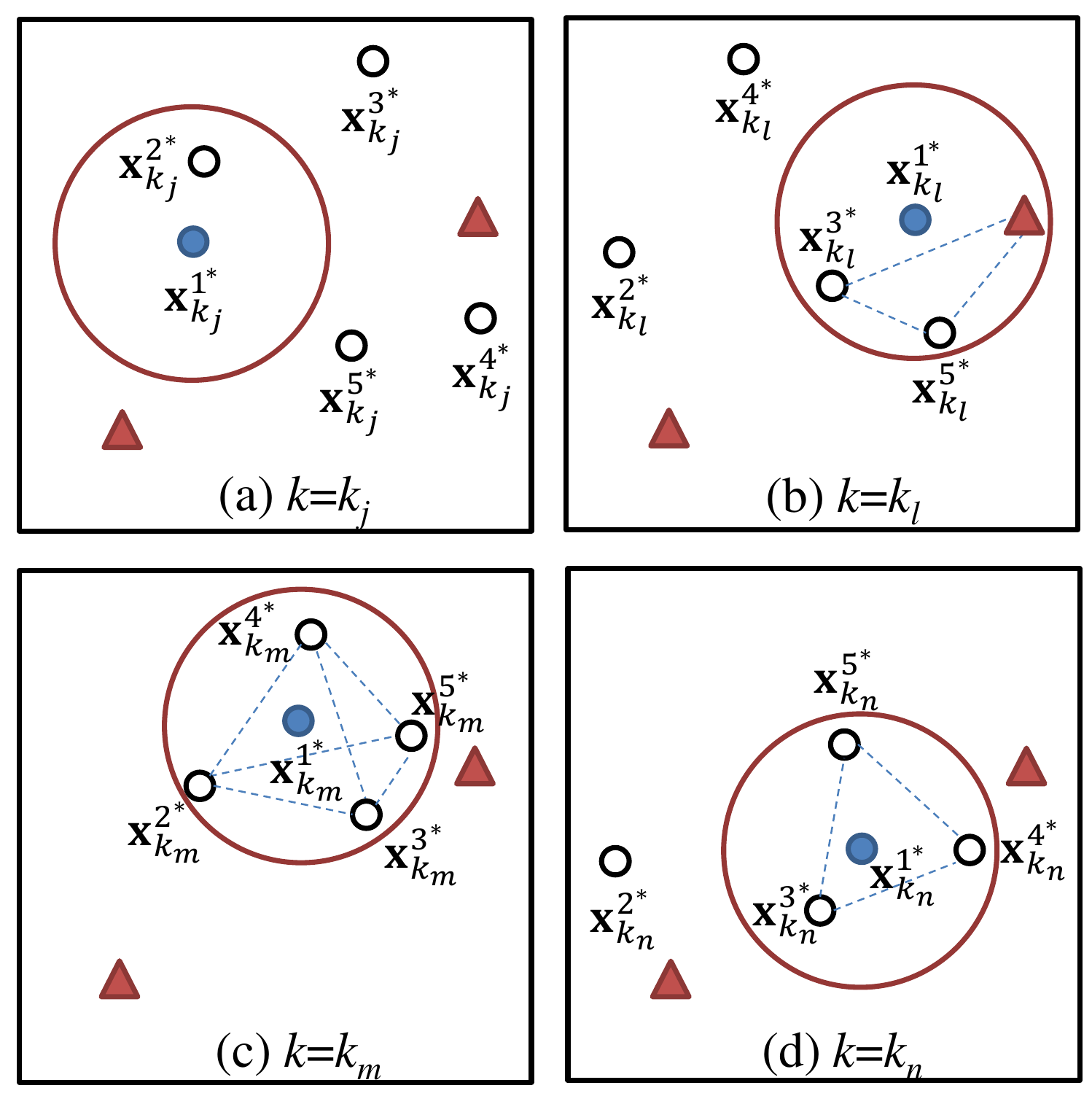}
	\caption{Update scenarios for a network of size $\vert \Psi \vert=7$; Red triangles represent beacons, and each circle represents a robot; robot $1$ is distinguished by a filled circle. The large circle shows the communication radius of robot~$1$.}
	\label{f0}
\end{figure}
\textcolor{black}{At any given time, $k$,} note the difference between~$\mc{N}_i(k)=\emptyset$ and~$\Theta_i(k)=\emptyset$; while the former \textcolor{black}{implies that node $i$ has no neighbors}, as in case (i), the latter implies that
robot~$i$ cannot find any triangulation set among the neighbors.\footnote{{\textcolor{black}{Note that our algorithm can be generalized to higher dimensions, e.g., $3$-D, where triangles and areas become tetrahedrons and volumes, respectively.}}} \textcolor{black}{For example at time~$k_j$, Fig.~\ref{f0} (a), we have~$\mc{N}_i(k)=\{2\}$ and~$\Theta_i(k)=\emptyset$, while at time $k_n$, Fig.~\ref{f0}~(d), we have~$\mc{N}_i(k)=\Theta_i(k)=\{3,4,5\}$.}
 
By separating the barycentric coordinates over the robots in~$\Omega$ and the beacons in $\kappa$, we can split Eq.~\eqref{18} as
\begin{eqnarray}\label{20}
\mb{x}^i_{k+1} &=& \alpha_k\mb{x}_k^i + (1-\alpha_k) \left(\sum_{j\in\Theta_i(k)\cap\Omega}\textcolor{black}{p_k^{ij}}\mb{x}_k^j\right),\nonumber\\ &+&(1-\alpha_k) \left(\sum_{m\in\Theta_i(k)\cap\kappa}\textcolor{black}{b_k^{im}}\mb{u}_k^{m\ast}\right)+\widetilde{\mb{x}}_{k+1}^i,\qquad\label{24}
\end{eqnarray}
where $\mb{u}_k^m$ is the true location of $m$-th beacon at time~$k$~and
\begin{eqnarray}
\textcolor{black}{a_k^{ij}} = 
\begin{cases}
\textcolor{black}{p_k^{ij}},& \mbox{if}~j \in \Theta_i(k)\cap\Omega,\\
\textcolor{black}{b_k^{im}},& \mbox{if}~m \in \Theta_i(k)\cap\kappa.
\end{cases}
\end{eqnarray}
It can be inferred from Eqs.~\eqref{alpk} and~\eqref{20} that \textcolor{black}{at time $k$ the self-weight assigned to the $i$-th robot's estimate, $p_k^{ii}$,} is always lower bounded, i.e.,
\begin{eqnarray}\label{B0eq}
0<\beta\leq \textcolor{black}{p_k^{ii}}\leq1,\qquad \forall k,i \in \Omega. 
\end{eqnarray}

Although it is possible that more than one robot finds a triangulation set and update at time $k$, in the remaining of the paper we assume, without loss of generality, that at most one robot updates at each iteration. We can now write the above algorithm in matrix form as follows:
\begin{eqnarray}\label{eq1}
{\bf{x}}_{{k+1}}={\bf{P}}_{{k}}{{\bf{x}}_{{k}}}+{\bf{B}}_{{k}}{\bf{u}}_{k}+\widetilde{{\bf{x}}}_{{k+1}},\qquad k>0,
\end{eqnarray}
in which~${\bf{x}}_{{k}} \in \mathbb{R}^{N\times m}$ is the vector of robot coordinates at time~$k$,~${\bf{u}}_{{k}} \in \mathbb{R}^{M\times m}$ is the vector of beacon coordinates at time~$k$, and~$\widetilde{{\bf{x}}}_{{k+1}}\in \mathbb{R}^{N\times m}$ is the change in the location of robots at the beginning of the~$k$-th iteration according to Eq.~\eqref{Eq1}. Note that ${\bf{P}}_{{k}}$, and~${\bf{B}}_{{k}}$, the system matrix and the input matrix of the above LTV system, contain 
(weighted) barycentric coordinates at time $k$, with respect to the robots with unknown locations, and beacons, respectively. 

Since true information is only injected into the network by the beacons, we must set a lower bound on the weights assigned to the beacon states. {\textcolor{black}{Otherwise, the beacons may be assigned a weight that goes to zero over time, i.e., the beacons eventually are excluded from the network.}} To make sure that beacons remain in the network, we naturally make the following assumption.

{\bf A3: Guaranteed beacon contribution}. \textcolor{black}{When a beacon is involved in an update, i.e., for any~$b_k^{im} \neq 0$, we impose~that}
\begin{eqnarray}\label{23}
0 < \alpha \leq \textcolor{black}{b_k^{im}},\qquad \forall k, i\in\Omega, {m\in\Theta_i(k)\cap\kappa},
\end{eqnarray}
\textcolor{black}{where $\alpha$ is the minimum beacon contribution, see Section~\ref{4c} (vi)}. 

The above assumption implies that if there is a beacon in the triangulation set, a certain amount of information is always contributed by the beacon. In other words, the update occurs only if the robot lies in an \textit{appropriate} location inside the convex hull. By Assumption {\bf{A3}}, the weight assigned to the beacon should be at least~$\alpha$. This weight comes from the  barycentric coefficient corresponding to the beacon,~i.e.,
\begin{eqnarray}\label{37}
\textcolor{black}{b_k^{im}} = \frac{A_{\Theta_i(k)\cup\{i\}\setminus m}}{A_{\Theta_i(k)}},\qquad m\in\kappa,
\end{eqnarray}
where $i$ and $m$ indices represent the updating robot and a beacon, respectively. To clarify Assumption {\bf{A3}}, we consider~Fig.~\ref{alpha} where the updating robot lies inside the convex hull of three nodes, including two other robots and one beacon; let~$\alpha=0.25$.
Fig.~\ref{alpha} (Left) shows a situation, where the robot updates, and \textcolor{black}{the beacon provides the exact} minimum contribution, i.e., the area of the shaded triangle is one fourth of the area of the convex hull triangle. 
\begin{figure}[!h]
	\centering
	\includegraphics[width=75mm]{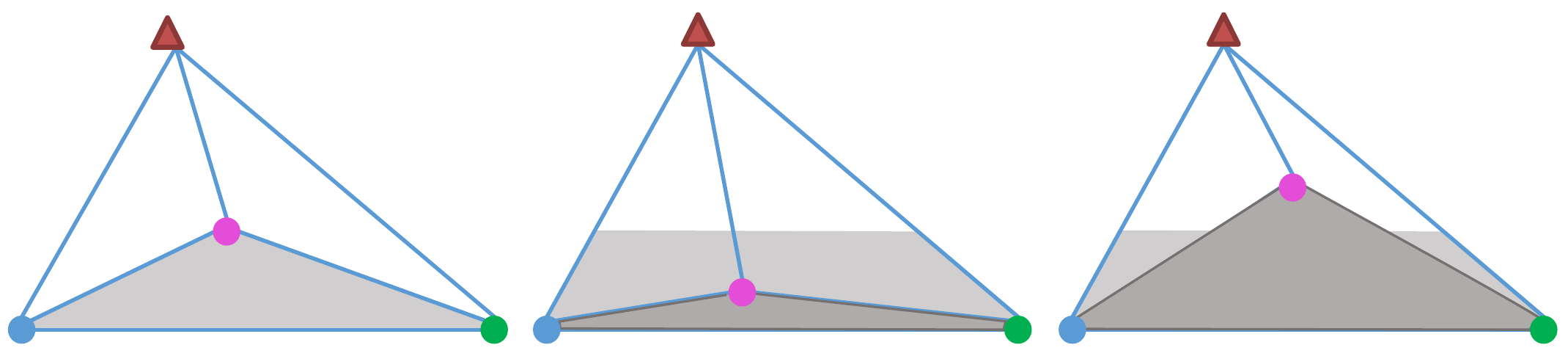}
	\caption{(Left) Robot is located on the threshold boundaries, where it updates and assigns the \textcolor{black}{exact} minimum weight, $\textcolor{black}{b^{ij}}=\alpha=0.25$, to the beacon; (Middle) Robot is located in an inappropriate location inside the convex hull, \textcolor{black}{no updates occurs}; \textcolor{black}{(Right) Robot updates with $b^{ij}>\alpha$}.}
	\label{alpha}
\end{figure}
On the other hand, if the robot lies within the shaded trapezoid illustrated in \textcolor{black}{Fig.~\ref{alpha}~(Middle)},
Eq.~\eqref{37} becomes less than $\alpha$, and Eq.~\eqref{23} does not hold. No update occurs in this case since the beacon does not provide enough valuable information for the robot. The upper side of the shaded trapezoid is the threshold boundary, such that if the robot lies on this line, an update occurs and the beacon provides the \textcolor{black}{exact} minimum contribution. Thus, our measure to evaluate the contribution of a beacon in an update, is whether or not the updating robot lies inside a specific area, and not the distance between the robot and the beacon. \textcolor{black}{If the updating robot stays inside the same convex hull, and moves closer to the beacon, as shown in Fig.~\ref{alpha} (Right), the robot performs an update and the weight assigned to the beacon exceeds $\alpha$.}
%
Note that in this example we set~$\alpha=0.25$ only for the ease of demonstration. The localization algorithm works for any arbitrary small (non-zero) value of $\alpha$, which corresponds to higher possibilities of a successful update when a beacon is involved. 

With the lower bounds on the self-weights according to Eq.~\eqref{B0eq}, and the lower bounds on the weights assigned to the beacons according to Eq.~\eqref{23}, the matrix of barycentric coordinates with respect to robots with unknown locations, i.e., the system matrix at time~$k$,~${\bf{P}}_{k}$, is either 
\begin{enumerate}[(i)]
	\item \emph{identity}, when no robot updates at time\textcolor{black}{\footnote{\textcolor{black}{Note that this case also characterizes the case, where a robot fails to update due to a (temporary) communication loss/drop.}}} $k$; or,
	\item \emph{identity except a stochastic~$i$-th row}, when a robot, say $i$, updates but finds no beacon in its triangulation set at time~$k$,~i.e.,$\Theta_i(k)\neq\emptyset$ and~$\Theta_i(k)\cap\kappa=\emptyset$;~or,
	\item \emph{identity except a sub-stochastic~$i$-th row}, when the updating robot, say $i$, has at least one beacon in its triangulation set at time~$k$,~i.e.,~$\Theta_i(k)\cap\kappa \neq \emptyset$. 
\end{enumerate}

 \noindent Note that when a robot, say $i$ at time~$k$, performs an update only with robots in the set~$\Omega$ (case (ii) above), the $i$-th row of the system matrix,~$\mathbf{P}_k$, contains all of the non-zero weights (corresponding to the barycentric coordinates) and~$\mb{B}_k$ is a zero matrix. Each row of $\mb{P}_k$ is identity except the $i$th row whose sums is one due to convexity. However, if a beacon in the set~$\kappa$ is involved in an update (case (iii) above), the $i$-th row of the input matrix,~$\mathbf{B}_k$, contains a non-zero weight assigned to that beacon. Therefore the weights assigned to the other neighboring robots sum to a value strictly less than one; in other words,~$\mb{P}_k$ is identity except the $i$-th row whose sum is strictly less than one making~$\mb{P}_k$ sub-stochastic. 

In the next section, we investigate the asymptotic stability of LTV systems where the system matrices follow the cases described above, and provide sufficient conditions for the iterative localization algorithm, Eq.~\eqref{eq1}, to converge to the true robot locations. Before we proceed, let us make the following definitions to clarify what we mean by stochasticity and sub-stochasticity throughout this~paper:
\begin{definition}
	A non-negative, stochastic matrix is such that all of its row sums are one. A non-negative, sub-stochastic matrix is such that it has at least one row that sums to strictly less than one and every other row sums to at most one.
\end{definition}

\section{Convergence analysis}\label{sec4}
We start this section by providing a related result on the convergence of an
infinite product of stochastic/sub-stochastic matrices,\textcolor{black}{~\cite{DBLP:journals/corr/SafaviK14}}.
We will then adapt these results to investigate the convergence of our localization algorithm represented by Eq.~\eqref{eq1}. 

\vspace{-0mm}

\subsection{Asymptotic behavior}\label{4a}
Consider the following LTV system:
\begin{equation}\label{14}
\mb{x}_{k+1}=\mb{P}_k\mb{x}_k,
\end{equation}
in which the time-varying system matrix,~$\mb{P}_k\triangleq\{p_k^{ij}\}$, is non-negative, random, and represents at most one state update, say at the~$i$-th row, for any~$k$. In other words,~$\mb{P}_k$'s can randomly switch between stochastic, and sub-stochastic matrices. In addition, assume the following on the update at time $k$: 

{\bf B0}: If $\sum_j p_k^{ij} =1$, i.e., the updating row,~$i$, in~$\mb{P}_k$~is stochastic, then the $i$-th self-weight is lower-bounded
as
\begin{eqnarray}\label{bnd1}
0 < \beta_1 \leq p_k^{ii},\qquad\beta_1\in\mbb{R}.
\end{eqnarray}

{\bf B1}: If $\sum_j p_k^{ij} < 1$, i.e., the updating row,~$i$, in~$\mb{P}_k$ is sub-stochastic, then it is also upper-bounded as
\begin{eqnarray}\label{bnd2}
\sum_jp_k^{ij} \leq\beta_2 < 1,\qquad\beta_2\in\mbb{R}.
\end{eqnarray}


To study the asymptotic behavior of an LTV system with such system matrices, in\textcolor{black}{~\cite{DBLP:journals/corr/SafaviK14}}, we introduce the notion of a slice,~$M_j$, as the smallest product of consecutive system matrices, 
such that: (i) the infinity norm of each slice is less than one, i.e.,~$\|M_t\|_\infty<1,\forall t$; \emph{and},
(ii) the entire sequence of system matrices is covered by non-overlapping slices,~i.e.,
%
\begin{equation}\label{mp}
\prod_t {{\bf{M}}}_{t} = \prod_{k} {\bf{P}}_{k}.
\end{equation}
Each slice initiates with a sub-stochastic system matrix, and terminates after all row sums become less than one, i.e., each row in the slice becomes sub-stochastic. The length of the $j$-th slice,~${\vert {M_j}\vert}$, is defined as the number of matrices forming the slice, and the upper bound on the infinity norm of a slice is further related to its length,
~\cite{7526779}, as 
\begin{equation}\label{43}
{{\Vert M_j \Vert}_{\infty}} \leq 1- {\beta_1}^{\vert {M_j}\vert - 1}{\beta}_2.
\end{equation} 

The following theorem characterizes the asymptotic behavior of the LTV system represented by Eq.~\eqref{14}:

\begin{thm}\label{thm0}
	With Assumptions~{\bf{B0-B1}}, the LTV system,~$\mb{x}_{k+1}=\mb{P}_k\mb{x}_k$, converges to zero
		, i.e., $\lim_{k\ra\infty}\mb{x}_k = \mb{0}_N$,
		if any one of the following is true:
		\begin{enumerate}[(i)]
			\item Each slice has a bounded length, i.e.,
			\begin{eqnarray}\label{51}
			\textcolor{black}{\vert M_j \vert \leq L < {\infty} , \qquad \forall j,~L \in \mathbb{N};}
			\end{eqnarray}
			
			\item \textcolor{black}{There exist an infinite subset,~$J_1$, of
			 slices such that}
			\begin{eqnarray}
			\vert M_{j} \vert \leq L_1 < {\infty}, \qquad \forall M_j\in J_1,\\
			\vert M_{j} \vert < {\infty}, \qquad \forall M_j \notin J_1;
			\end{eqnarray}
			
			\item \textcolor{black}{For every~$i \in \mathbb{N}$, there exists an infinite subset,~$J_2$, of slices such that} 
			\begin{equation}\label{growth}
			\exists M_j \in J_2:~~\vert M_j \vert \leq \frac{1}{\ln\left({\beta_1}\right)}\ln\left(\frac{1 - e^{(-\gamma_2i^{-\gamma_1})}}{1-\beta_2}\right)+1,
			\end{equation}
			for some~$\gamma_1 \in [0,1]$, $\gamma_2>0$. For any other slice,~$M_j,~j\notin J_2$ we have~$|M_j|<\infty$. 
		\end{enumerate}
\end{thm}
We now briefly discuss the intuition behind the above theorem. According to Eq.~\eqref{mp}, we can use the product of slices instead of the product of system matrices, and study the following dynamics rather than Eq.~\eqref{14}: ${\bf{y}}({t+1})={\bf{M}}_t {\bf{y}}({t})$. By taking the infinity norm of the both sides of this equation and using the sub-multiplicative norm property we get
\begin{align}\label{55}
{\Vert{\bf{y}}(t+1)\Vert}_\infty &\leq {\Vert {{M}}_{t} \Vert}_{\infty}  \ldots {\Vert {{M}}_{0} \Vert}_{\infty} {\Vert {\bf{y}}(0) \Vert}_{\infty}.
\end{align}

The first two cases are trivial, because if all or an infinite subset of slices have a bounded length, then according to Eq.~\eqref{43} each slice will have a bounded (and subunit) infinity norm, whose infinite product is zero. Thus, from Eq.~\eqref{55} we can infer that ${\Vert{\bf{y}}(t+1)\Vert}_\infty=0$, which in turn leads to~$\lim_{k\ra\infty}\mb{x}_k = \mb{0}_N$,
and completes the proof. However, we can show that a strict upper bound on the lengths of all or an infinite subset of slices is not necessary. In fact, all we require is an infinite subset of slices whose (unbounded) lengths grow slower than the exponential rate provided in Eq.~\eqref{growth}. Later in this section, we explain how the slice lengths can be interpreted in the context of inter-network communications. The reader is referred to our prior works,~\cite{DBLP:journals/corr/SafaviK14,7526779}, for a detailed proof of Theorem~\ref{thm0}. 
In what follows, we use the results of this theorem to study the convergence of Eq. \eqref{eq1}.

\subsection{Convergence of the localization algorithm}\label{4b}
We start this section with the following lemma:
\begin{lem}\label{lem5}
	Under Assumptions {\bf{A0-A3}} and no noise, 
	the product of system matrices,~$\mb{P}_k$'s, in 
	the LTV system represented by Eq.~\eqref{eq1}, converges to zero if any one of the three conditions in Theorem~\ref{thm0} holds.
\end{lem}
\begin{proof}
	We need to show that~{\bf B0} and {\bf B1} can be inferred from Assumptions~{\bf A0-A3}. 
	First note that Eq.~\eqref{B0eq} results in~{\bf B0} if we set $\beta_1=\beta$. On the other hand, if Assumption~{\bf A3} holds, i.e., if there exist at least one beacon in the triangulation set,
	we can write
	\begin{eqnarray}\label{30}
	\sum_{j\in\Theta_i(k)\cap\Omega}\textcolor{black}{p_k^{ij}} = 1-\sum_{m\in\Theta_i(k)\cap\kappa}\textcolor{black}{b_k^{im}},
	\end{eqnarray}
	in which we used the fact that barycentric coordinates sum to one.
	Note that the second term on the right hand side (RHS) of Eq.~\eqref{30} gives the sum of all weights assigned to the beacon(s) in the triangulation set. This term is minimized, and hence the RHS of Eq.~\eqref{30} is maximized, when there is only one beacon among the neighbors of the updating robot, and the minimum weight, $\alpha$, is assigned to this beacon according to Eq.~\eqref{23}. In this case we can write Eq.~\eqref{30} as
	\begin{eqnarray}\label{31}
	\sum_{j\in\Theta_i(k)\cap\Omega}\textcolor{black}{p_k^{ij}} \leq 1-\alpha <1,
	\end{eqnarray}
	which provides an upper bound on the~$i$-th row sum of~${\bf{P}}_{{k}}$, and results in Assumption~{\bf B1} if we choose~$\beta_2=1-\alpha$. With Assumptions~{\bf B0-B1} satisfied, the rest of the proof follows from that of Theorem~\ref{thm0}.
\end{proof}

We now provide our main result in the following theorem.

\begin{thm}\label{th2}
	\textcolor{black}{Consider a network of $M$ (possibly mobile) beacons and $N$ mobile robots moving in
	a finite and bounded region. Then, under Assumptions {\bf A0-A3} and no noise, for any (random or deterministic)
	motion that satisfies one of the conditions in Theorem~\ref{thm0},
 the solution of Eq.~\eqref{eq1} asymptotically converges to the true robot locations.} 
\end{thm}
\begin{proof}
	The  motion of the robots in a finite,
	bounded region results in the following LTV system:
	\begin{eqnarray}\label{eq1j}
	{\bf{x}}_{{k+1}}={\bf{P}}_{{k}}{{\bf{x}}_{{k}}}+{\bf{B}}_{{k}}{\bf{u}}_{k}+\widetilde{{\bf{x}}}_{{k+1}},\qquad k>0,
	\end{eqnarray}
	in which the system matrices, ${\bf{P}}_{{k}}$'s switch between stochastic and sub-stochastic matrices. Therefore, under Assumptions {\bf A0-A3} and according to Lemma~\ref{lem5}, \textcolor{black}{if any of the conditions in Theorem~\ref{thm0} is satisfied}, we have
	\begin{equation}\label{29}
	\lim_{k \rightarrow \infty} \prod_{l=0}^{k} {\bf{P}}_{l}=\mb{0}_{N\times N}.
	\end{equation}  
	On the other hand, the true robot locations are given by:
	\begin{eqnarray}\label{eq2}
	{\bf{x}}^{*}_{{k+1}}={\bf{P}}_{{k}}{{\bf{x}}^{*}_{{k}}}+{\bf{B}}_{{k}}{\bf{u}}_k+\widetilde{{\bf{x}}}_{{k+1}}.
	\end{eqnarray}
	To justify the above equation, note that if robot $i$ lies inside the convex hull of $m+1$ neighbors at time $k$, its true location can be expressed as a convex combination of the \textit{true} locations of its neighbors. \textcolor{black}{Note that when the inclusion test is not passed at any robot at time $k$, the system matrix, ${\bf{P}}_{{k}}$, and the input matrix,~${\bf{B}}_{{k}}$, become identity and zero matrices, respectively, and the location estimate,~${\bf{x}}_{{k+1}}$, and the true locations,~${\bf{x}}^{*}_{{k+1}}$ in Eqs.~\eqref{eq1j} and \eqref{eq2} update according to the motion vector.} By subtracting Eq.~\eqref{eq1j} from Eq.~\eqref{eq2} the error dynamics can be obtained as follows
	\begin{eqnarray}\label{eq3}
	\mb{e}_{k+1}\triangleq{\bf{x}}^{*}_{{k+1}}-{\bf{x}}_{{k+1}}={\bf{P}}_{{k}} ({\bf{x}}^{*}_{{k}}-{\bf{x}}_{{k}})={\bf{P}}_{{k}}\mb{e}_{{k}},
	\end{eqnarray}
	which converges to zero from Lemma~\ref{lem5} and Eq.~\eqref{29}, and the proof is complete.
\end{proof}

{\vspace{-4mm}\subsection{Discussion}\label{4c}}
\textcolor{black}{In what follows, we shed some light on Theorem~\ref{thm0}, and elaborate on the choice of design parameters, $\alpha$ and $\alpha_k$.}

(i) As mentioned earlier in Section~\ref{4a}, each slice is initiated with a sub-stochastic update, which only occurs when a beacon is among the neighbors of the updating robot. In other words, to initiate a slice a robot with unknown location must have a beacon in its triangulation set. Thus, the first system matrix of each slice has one sub-stochastic, and $N-1$ stochastic rows. We call a robot \emph{informed} after it updates with a beacon. Therefore, at the beginning of a slice there are one informed and $N-1$ \emph{uninformed} robots in the network.

(ii) A slice is terminated after all rows of a slice are sub-stochastic, i.e., all robots are informed. For this to happen, each robot has to either update with a beacon directly, or indirectly, i.e., update with a robot who updated with a beacon in the same slice. By \emph{indirectly}, we mean that a robot receives information not from a beacon but by any other informed~robot in the network.
%
Completion of each slice corresponds to~the propagation of location information from beacon(s) to all robots in the network, and the location estimates are refined each and every time a slice is completed. Hence, slice representation is closely related to the information flow from the beacons to every other robot in the network. 

(iii) Theorem~\ref{thm0} provides the conditions on the rate,
at which such information should propagate for convergence to the true robot locations. The first two cases in Theorem~\ref{thm0} require all or an infinite subset of slices to have bounded lengths. In other words, the information from the beacons has to reach all robots within $L$ or $L_1$ iterations, infinitely often. These conditions are relaxed in the third case; Eq.~\eqref{growth} implies that all we need is an infinite subset of slices whose lengths grow slower than a certain exponential rate. The condition on motion described in Theorem~\ref{th2} can be translated to a condition on the information dissemination from Theorem~\ref{thm0}. \textcolor{black}{Thus, any motion that guarantees this information dissemination suffices.} 

\textcolor{black}{(iv) The non-zero self-weights, $\alpha_k$'s, assigned to the previous state of the updating robot guarantees that an informed robot does not become uninformed again within the same slice, e.g., by performing an update with a set of uninformed robots before the slice is complete.} 

\textcolor{black}{(v) Although the proposed algorithm converges for any value of~$0< \alpha <1$, the convergence rate of the algorithm is affected by the choice of~$\alpha$: By choosing~$\alpha$ arbitrarily close to $1$, a robot has to get arbitrarily close to a beacon in order to perform an update with respect to that beacon (see Fig.~\ref{alpha}), which in turn corresponds to arbitrary large number of iterations for the termination of each slice. On the other hand, setting~$\alpha$ arbitrarily close to zero makes slice norms arbitrarily close to $1$. 
A proper choice can be made by considering the motion model, the communication protocol, and the number of available beacons in the network.}

\vspace{3mm}

\section{Localization under imperfect measurements}\label{noise}
The noise on the motion and distance measurements degrades the performance of the localization algorithm, as expected, and in certain cases the location error is larger than the region of motion; this is shown experimentally in Section~\ref{sec7}. In what follows, we provide the modifications, {\bf{M1-M3}}, to the proposed algorithm to counter the undesirable effects of noise in case of motion,~${\widehat{\mb{x}}_{k}^i}$, and on the distance measurements,~$\widehat{d}^{ij}_k$. 

\vspace{-5mm}

\textcolor{black}{\subsection{Discard unreliable Cayley-Menger determinants}
If a robot is located close to the boundaries of a convex hull, the noise on distance measurements may affect the inclusion test results. 
To get meaningful values for the areas and volumes in~$\mathbb{R}^2$ and $\mathbb{R}^3$ according to Eq.~\eqref{cmeq},
the corresponding Cayley-Menger determinants computed with perfect distance measurements must obey a certain sign in order for the square root to convey a meaningful result (negative in $\mbb{R}^2$ and positive in $\mbb{R}^3$). Therefore, we suggest the first modification to the algorithm as follows:}

\textcolor{black}{{\bf{M1:}} A robot does not perform an inclusion test if the corresponding Cayley-Menger determinant is positive in $\mathbb{R}^2$, or negative in $\mathbb{R}^3$.}

\subsection{Inclusion test error}
Even in the case of perfect distance measurements, the inclusion test results may not be accurate due to the noise on the motion, which in turn corresponds to imperfect location updates \textit{at each and every iteration}. To address this issue, we propose the following modification to the algorithm:
	
	{\bf{M2:}} Suppose the inclusion test is passed at time $k$ by a triangulation set, $\Theta_i(k)$. Robot $i$ performs an update only if 
	\begin{eqnarray}\label{iterror}
	\textcolor{black}{\epsilon_k^{i}=\left\vert\frac{\sum_{j\in\Theta_i(k)}A_{\Theta_i(k)\cup\{i\}\setminus j} - A_{\Theta_i(k)}}{A_{\Theta_i(k)}}\right\vert<\epsilon,}
	\end{eqnarray}
	where~$\epsilon_k^{i}$ is the \textit{relative inclusion test error} at time~$k$ for robot~$i$, and $\epsilon$ is a design parameter. \textcolor{black}{The optimal choice of $\epsilon$ depends on the number of robots and beacons in the network and the statistics of the noise.} 	
\vspace{-2mm}	
	
	\subsection{Convexity}
	Finally, in order to guarantee the convexity in the updates \textcolor{black}{in the presence of noise}, we consider the following modification:
	
	{\bf{M3:}} Suppose the inclusion test is passed by a triangulation set, $\Theta_i(k)=\{j,~l,~m\}$, and Eq.~\eqref{iterror} holds at time $k$. \textcolor{black}{Robot~$i$ first computes, $a_{k}^{{i}j}$, $a_{k}^{{i}l}$, and $a_{k}^{{i}m}$, according to Eq.~\eqref{eq7} and using the noisy distance measurement. It then normalizes the weights assigned to each neighbor in order to preserve the convexity of the update. For example, the (normalized) weight assigned to robot $j$ is computed as~$a_{k}^{{i}j}/\sum_{n\in{\Theta_i(k)}}a_{k}^{{i}n}$.}
	
	\vspace{2mm}
	
In Section~\ref{sec7}, we show that the above modifications to the algorithm improve the results significantly in presence of noise, and result in a bounded error in location estimates.

\section{How many beacons are necessary?}\label{sec5}
In this section, we investigate the minimal number of beacons required for localizing an arbitrary number of robots using the algorithm described in Section~\ref{sec3}. We denote the subspace of motion at robot, $i \in\Omega$, and beacon, $j \in \kappa$, by~$\mathcal{M}_i$ and $\mathcal{U}_j$, respectively. Let us clarify this notation with a simple example: Suppose robot $1$ is moving along a vertical line \textcolor{black}{(regardless of the direction)}; this line forms~$\mathcal{M}_1$, and $\dim \mc{M}_1 =1$. \textcolor{black}{Note that $\mc{M}_i$ or~$\mathcal{U}_j$ includes all possible locations that the $i$-th robot or the $j$-th beacon occupies throughout the localization process, i.e., discrete times~$k=1,~2,\ldots$}.
Now consider another robot, $2$, which is moving along a vertical line parallel to~$\mathcal{M}_1$ \textcolor{black}{(regardless of the direction)}; in this case we will have~$\dim \cup_{i=1,2} \mathcal{M}_i=\dim\mc{M}_2=1$. However, if the two lines are linearly independent, they span~$\mathbb{R}^{2}$, and we will have~$\dim \cup_{i=1,2} \mathcal{M}_i=2$. 

Assuming~$\mathbb{R}^{m}$,~we show that the motion of the robots and beacons in~$l\leq m$ dimensions allows us to reduce the number of beacons from~$m+1$ by~$l$. \textcolor{black}{Note that the traditional trilateration scheme requires at least~$3$ nodes with known locations in $\mathbb{R}^2$. Therefore, assuming $m+1$ beacons in~$\mathbb{R}^m$ has been standard in all multilateration-based localization algorithms in the literature, see e.g.,~\cite{thomas2005revisiting,navarro1999beacon}.}
%
%
\textcolor{black}{In the following theorem, we develop necessary conditions for the proposed algorithm to track the true location of robots.
Since we can provide robots with up to $m$ degrees of freedom in their motion in~$\mathbb{R}^{m}$, we then show that our localization algorithm works in the presence of only one~($m+1-m$) beacon.} 
\begin{thm}\label{thm3}
For the LTV dynamics in Eq.~\eqref{eq1} to track the true robot locations, \textcolor{black}{following a non-trivial configuration in~$\mathbb{R}^m, m=2$}, the following conditions must be satisfied
\begin{eqnarray}
\vert \kappa \vert &\geq& 1,\label{33}\\
\vert \kappa \vert + \vert \Omega \vert &\geq& m+2,\label{34}\\
\vert \kappa \vert + \dim \underset{i\in\Omega}\cup \mc{M}_i + \dim \underset {j\in\kappa} \cup \mc{U}_j &\geq& m+1 .\label{35}
\end{eqnarray}
\end{thm}

\begin{proof}
\textcolor{black}{Eq.~\eqref{33} is trivial, because when there is no beacon in the network, all system matrices become stochastic, and the error dynamics in Eq.~\eqref{eq3} do not converge to zero. \textcolor{black}{This is equivalent to Case (ii) in Section~\ref{sec3}, where the update is always in terms of robots, i.e.,~$\mb{B}_k$ is always zero in Eq.~\eqref{eq1}.} Eq.~\eqref{34} stems from the fact that each robot requires at least~$m+1$ neighbors to perform an update in $\mathbb{R}^m$, \textcolor{black}{assuming non-trivial configurations.} In order to prove the necessary condition in Eq.~\eqref{35}, let us first consider the $2$-dimensional Euclidean space. Suppose on the contrary that Eq.~\eqref{35} does not hold.
In this case, there exist three possible scenarios as follows:
\begin{enumerate}[(i)]
	\item $\vert \kappa \vert=0, \mbox{and}\dim \underset{i\in\Omega}\cup \mc{M}_i=2$; or,
	\textcolor{black}{\item $\vert \kappa \vert=1, \mbox{and}~\dim \underset{i\in\Omega}\cup \mc{M}_i + \dim \underset {j\in\kappa} \cup \mc{U}_j =0$; or,}
	\item $\vert \kappa \vert=1, \mbox{and}~\dim \underset{i\in\Omega}\cup \mc{M}_i + \dim \underset {j\in\kappa} \cup \mc{U}_j =1$; or,
	\item $\vert \kappa \vert=2, \mbox{and}~\dim \underset{i\in\Omega}\cup \mc{M}_i + \dim \underset {j\in\kappa} \cup \mc{U}_j =0$.
\end{enumerate}
Localization is not possible in case (i), as it violates Eq.~\eqref{33}.} 

\textcolor{black}{Case (ii)--with~$1$ beacon,~$3$ robots and no motion, is where at most one robot may be able to lie inside the convex hull of the remaining three; making the triangulation of the second robot impossible. Same argument can be applied for Case (iv)--with~$2$ beacons,~$2$ robots and no motion. Mathematically, these two cases mean that a slice is never completed because the rows of $\mb{P}_k$ in Eq.~\eqref{eq1} corresponding to the non-updating robots will always be the corresponding rows of identity.} Thus, all we need to show is that localization is not possible in case (iii), where there is one beacon in the network, and the dimension of the motion in all nodes is one, i.e., all nodes can only move along parallel lines in $\mathbb{R}^2$.

We start with the best possible scenario, which initially allows one robot to triangulate. 
\begin{figure}[!h]
	\centering
	\includegraphics[width=65mm]{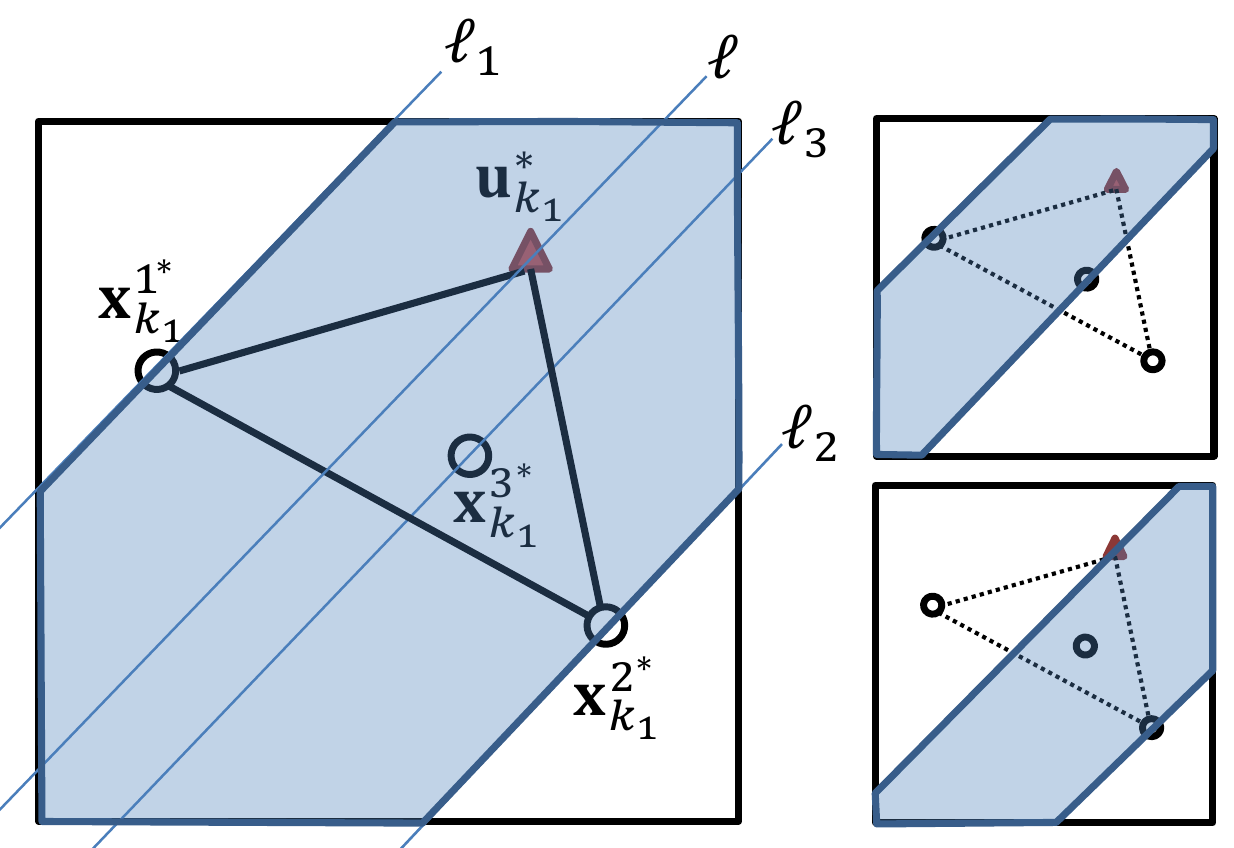}
	\caption{\textcolor{blue}{Only robot $1$ can triangulate in this configuration.
		}}
		\label{f6j}
	\end{figure}
We then choose an arbitrary direction in which all nodes are allowed to move. As shown in Fig.~\ref{f6j} (Left), the beacon and the three robots can move along~$\ell$,~$\ell_1$,~$\ell_2$, and $\ell_3$, respectively.
If this direction is chosen such that the beacon's line of motion lies within the strip between $\ell_1$ and $\ell_2$ (the lines of motion for the two robots that are not initially able to triangulate), no other robot will be able to triangulate. That is because robots $1$ and $2$ have to move inside the shaded regions illustrated in Fig.~\ref{f6j} (Right) in order to (possibly) triangulate. Clearly, any vector that provides such motion is linearly independent of the robots' direction of motion.

\textcolor{black}{We now consider the case where the direction of motion is such that the beacon lies outside the aforementioned strip. This scenario is illustrated in Fig.~\ref{f7}.
	\begin{figure}[!h]
		\centering
		\includegraphics[width=75mm]{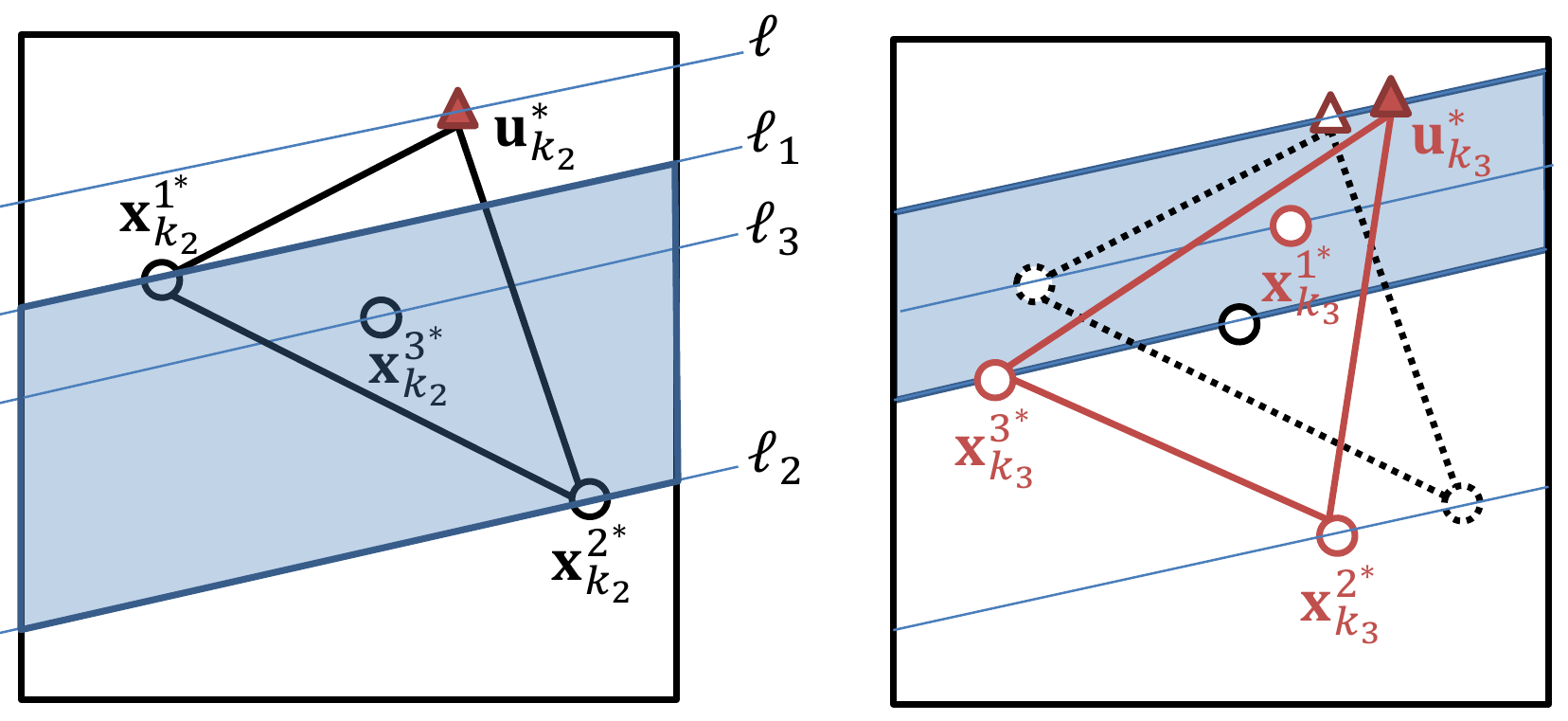}
		\caption{\textcolor{black}{(Left) Robot $3$ triangulates at time $k_2$; (Right) Robot $1$ triangulates at time $k$; Robot $2$ can not triangulate in this configuration.}}
		\label{f7}
	\end{figure}
	At time $k_2$, robot $3$ lies inside the convex hull of the beacon and the other two robots, and is able to triangulate, see Fig.~\ref{f7} (Left). As shown in Fig.~\ref{f7} (Right), the robots can move along the specified parallel lines such that at time $k_3$ robot $1$ lies inside the triangle of the beacon and robots $2$ and $3$. However, the farthest robot from the beacon (robot $2$ in this case) will never be able to triangulate, unless it moves inside the shaded region shown Fig.~\ref{f7} (Right), i.e., the strip between the lines of motion corresponding to robot $3$ and the beacon. This in turn requires a motion vector that is linearly independent of the robots' specified direction of motion, i.e., $\dim \underset{i\in\Omega}\cup \mc{M}_i$ greater than~$1$.}
{At time $k_2$, Fig.~\ref{f7} (Left), robot $3$ lies inside the convex hull of the beacon and the other two robots, and is able to triangulate. Due to the motion, Fig.~\ref{f7} (Right), the robots move along the specified parallel lines such that at time $k_3$, robot $1$ lies inside the triangle of the beacon and robots $2$ and $3$. However, the farthest robot from the beacon (robot $2$ in this case) will never be able to triangulate, unless it moves inside the shaded region shown in Fig.~\ref{f7} (Right), i.e., the strip between the lines of motion corresponding to robot $3$ and the beacon. This in turn requires a motion vector that is linearly independent of the robots' specified direction of motion, i.e., $\dim \underset{i\in\Omega}\cup \mc{M}_i$ greater than~$1$. Mathematically, row~$2$ of $\mb{P}_k$ in Eq.~\eqref{eq1} never changes and a slice never completes.}
\end{proof}

\textcolor{black}{Although Theorem~\ref{thm3} is applicable to~$m=3$, a formal proof is beyond the scope of this paper and will be provided elsewhere.} In the sequel, we only give a sketch of the proof to show that localization is not possible with one beacon and two dimensional motion in~$\mbb{R}^3$. This scenario is shown in Fig.~\ref{f8}, where robot $4$ initially lies inside the convex hull (a tetrahedron in $\mathbb{R}^3$) of the other nodes. 
\begin{figure}[!h]
	\centering
	\includegraphics[width=55mm]{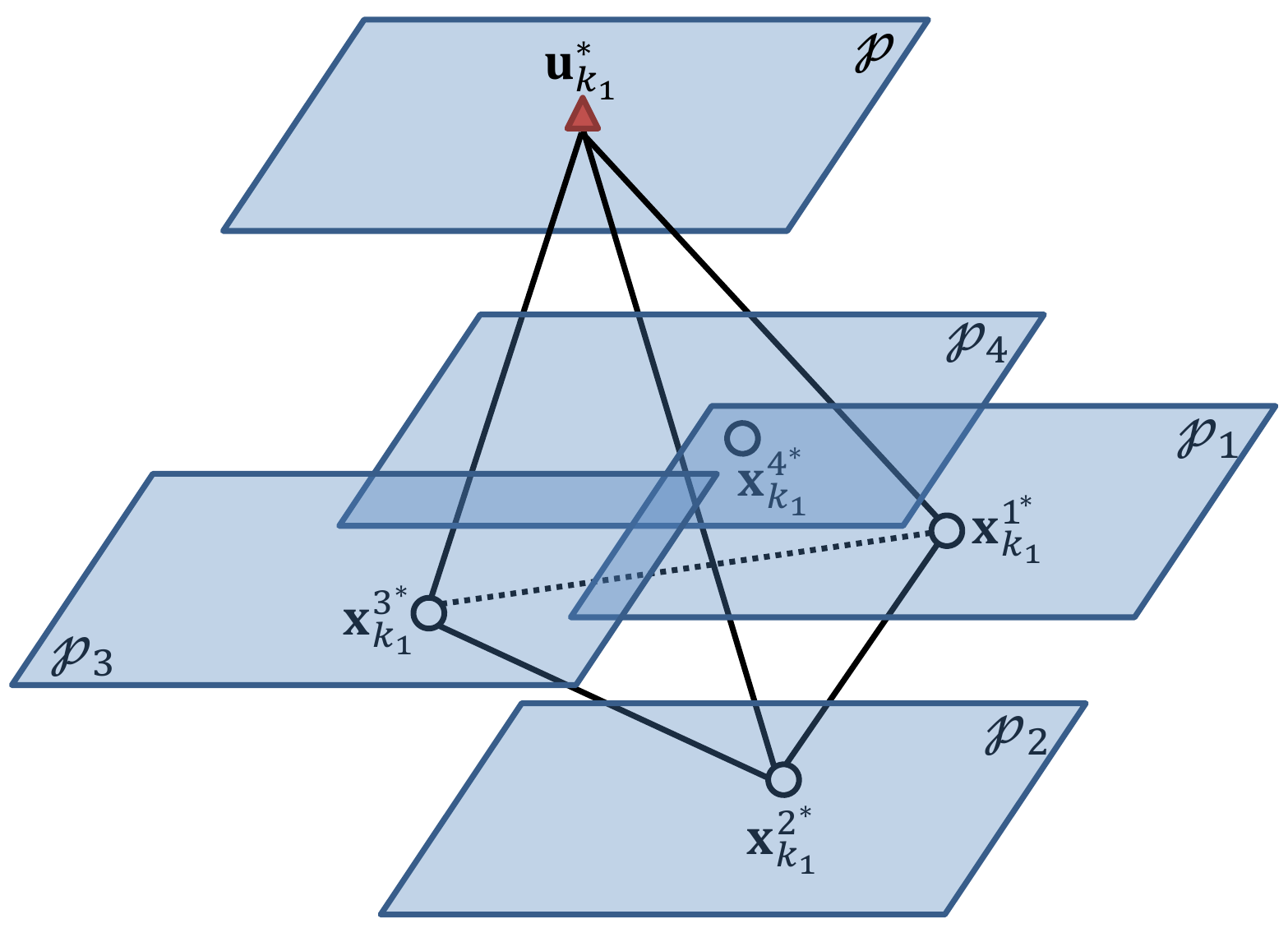}
	\caption{Localization with one beacon and $2$-dimensional motion is not possible in $\mathbb{R}^3$.}
	\label{f8}
\end{figure}
The beacon and robots $1-4$ can move on the planes $\mathcal{P}$, and $\mathcal{P}_1-\mathcal{P}_4$ as shown in Fig.~\ref{f8}. Clearly, the farthest robot from the beacon, robot $2$ in this case, is not able to find a triangulation set, unless it moves into the space between $\mathcal{P}_3$ and $\mathcal{P}$, which in turn requires a motion vector which is not in the span of any pair of linearly independent vectors in the motion planes, and thus makes localization of robot $2$ impossible.

Theorem \ref{thm3} provides necessary conditions for our localization algorithm in terms of the minimal number of beacons and the dimension of motion in the network.  
\begin{figure}[!h]
	\centering
	\includegraphics[width=75mm]{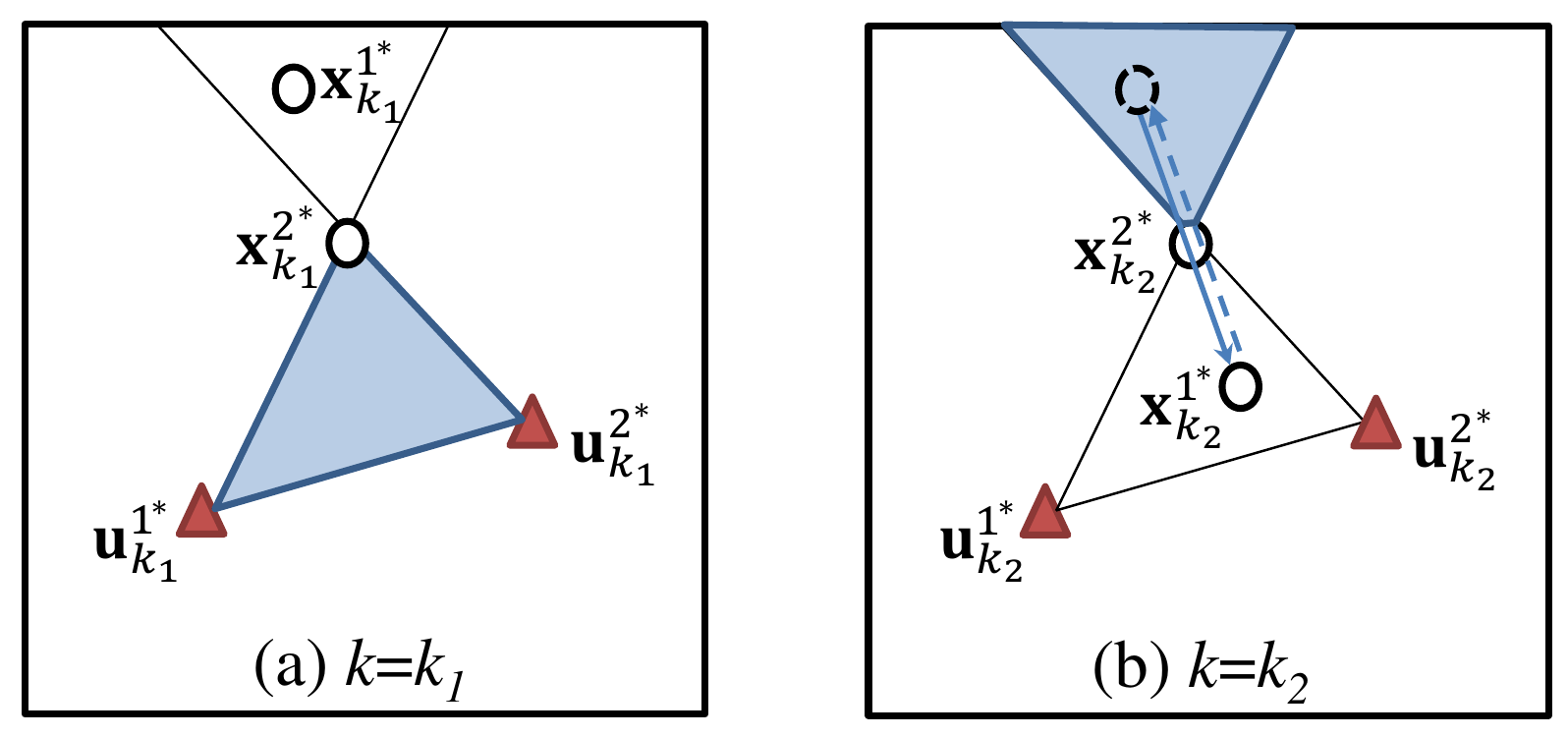}
	\caption{Localization with two beacons and one dimensional motion \textcolor{black}{in $\mathbb{R}^2$.}}
	\label{f4}
\end{figure} 
In what follows, we show how adding a new beacon or a new dimension in the motion make the localization algorithm work in $\mathbb{R}^2$. We start with case (ii) in Theorem~\ref{thm3} and add one more beacon, i.e.,~$\vert \kappa \vert=2, \mbox{and}~\dim \underset{i\in\Omega}\cup \mc{M}_i + \dim \underset {j\in\kappa} \cup \mc{U}_j=1$. This scenario is illustrated in Fig.~\ref{f4}.
Without loss of generality, we assume that only one robot, $1$, is moving from time~$k_1$ to $k_2$. To triangulate, this robot has to lie inside the convex hull of the two beacons and robot $2$, the shaded region in Fig.~\ref{f4} (a). To this aim, one possible motion vector is shown in Fig.~\ref{f4} (b) with a solid vector. Note that in order to let robot $2$ triangulate, robot $1$ can move via the same motion vector (and in the opposite direction) to lie inside the shaded region shown in Fig.~\ref{f4}.

We now revisit case (ii) in Theorem~\ref{thm3} and add one more motion dimension, i.e.,~$\vert \kappa \vert=1, \mbox{and}~\dim \underset{i\in\Omega}\cup \mc{M}_i + \dim \underset {j\in\kappa} \cup \mc{U}_j=2$. 
\begin{figure*}
	\centering
	\subfigure{\includegraphics[width=1.4in,height=1.4in]{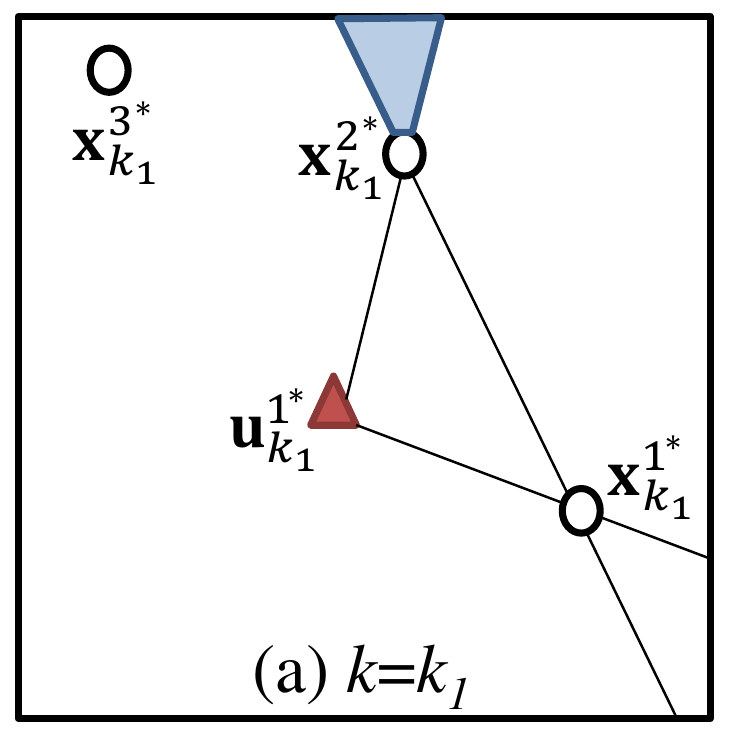}}
	\subfigure{\includegraphics[width=1.4in,height=1.4in]{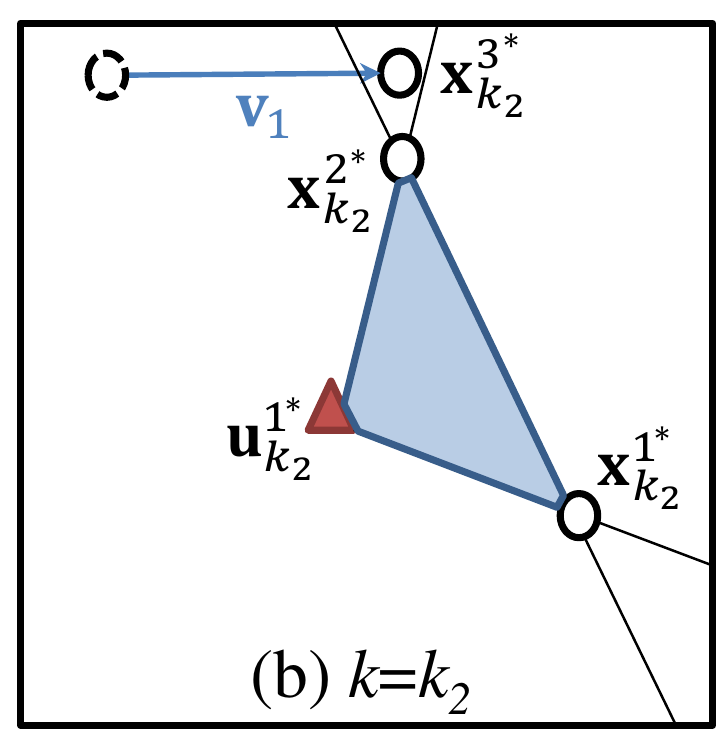}}
	\subfigure{\includegraphics[width=1.4in,height=1.4in]{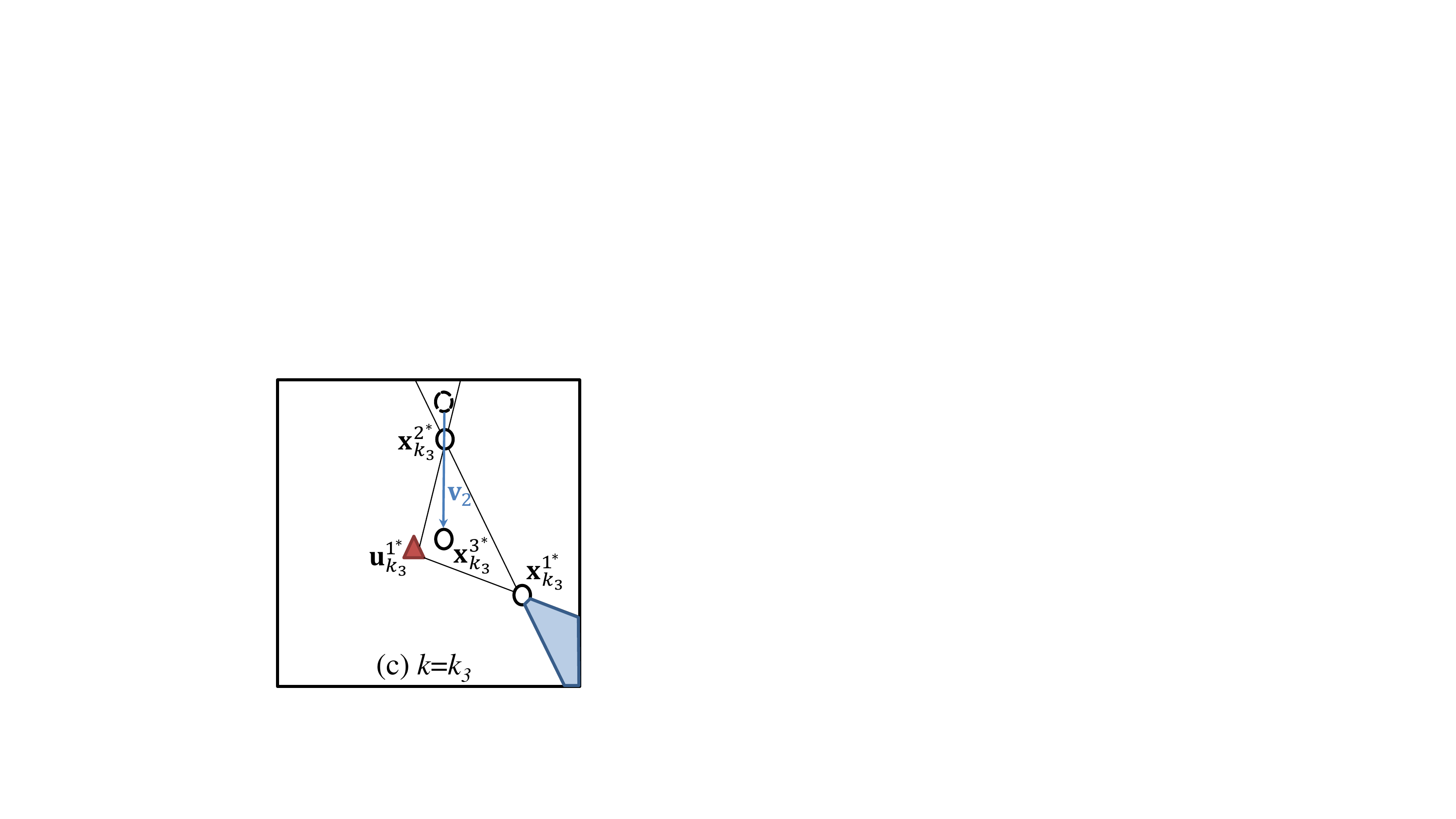}}
	\subfigure{\includegraphics[width=1.4in,height=1.4in]{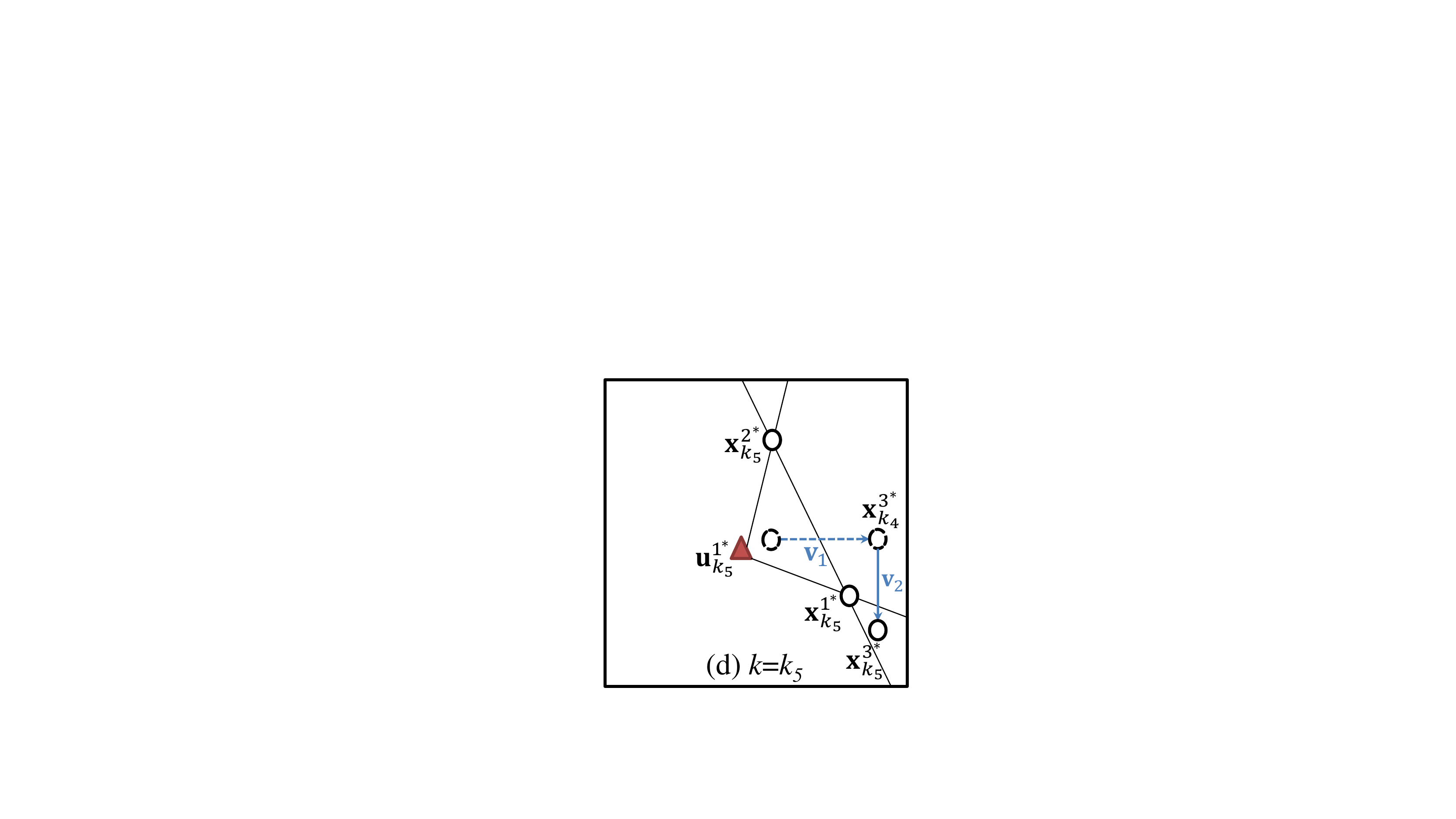}}
	\caption{\textcolor{black}{Localization with one beacon and two dimensional motion; Robot~$2$, Robot~$3$, and Robot~$1$ triangulate at times $k_2$, $k_3$ and $k_5$, respectively.}}
	\label{f5-1}
\end{figure*}
A possible motion trajectory of one robot, $1$, which leads to the triangulation of all robots in a network of size $4$ is shown in Fig.~\ref{f5-1}, assuming that the beacon and all the robots, except, robot $3$, are static. We denote the motion vectors with~${\bf{v}_1}$ and ${\bf{v}_2}$, which are orthogonal in this case.
\textcolor{black}{The shaded regions in Fig.~\ref{f5-1} (a), (b), and (c) represent the locations, where robots~$4$,~$1$, and $3$ can triangulate. For example, Robot $2$ lies inside a triangle of Robots $3$, Robot $1$, and the beacon, only if Robot~$3$ moves to the shaded region in Fig.~\ref{f5-1} (a) between time $k_1$ and $k_2$.}
\textcolor{black}{Note that Fig.~\ref{f5-1} illustrates a scenario, in which one slice is completed, i.e., each robot receives information from the beacon at least once. Clearly, to achieve the convergence to the true locations, such motion has to be repeated infinitely often such that one of the conditions in Theorem~\ref{thm0} is satisfied.}

\section{Simulations}\label{sec7}
In this section, we provide the simulation results to illustrate the proposed localization algorithm in $\mathbb{R}^2$. 
Let the true location of the~$i$-th robot,~$i\in\Omega$ in~$\mbb{R}^2$, be decomposed as~${\mb{x}_k^{i\ast}}=[{x_k^{i\ast}}~{y_k^{i\ast}}]$. We consider the following random motion model for robot~$i \in \Omega$:
\begin{equation}\label{1}
{x_{k+1}^{i\ast}}= {x_{k}^{i\ast}}+d_{k+1}^i\cos(\theta_{k+1}^i),~~
{y_{k+1}^{i\ast}}= {y_{k}^{i\ast}}+d_{k+1}^i\sin(\theta_{k+1}^i),
\end{equation}
in which~$d^i_{k+1}$ and~$\theta^i_{k+1}$ denote the distance and angle traveled by robot~$i$, between time~$k$ and~$k+1$, and are random. \textcolor{black}{We choose,~$d^i_{k}$ and~$\theta^i_{k}$, to have uniform distribution over the intervals of~$[0~d_{\max}]$ and~$[0~2\pi]$, respectively, such that each robot,~$i\in\Omega$, does not leave the bounded region of interest.} Also note that the random motion is assumed to be statistically independent among the robots. In what follows, we first provide the simulation results in noiseless scenarios, and then study the effects of noise and the proposed modifications in the convergence of the algorithm.

\subsection{Localization in noiseless scenarios}
We first consider a network of $5$ mobile robots with unknown locations, and only one beacon, whose location is fixed and perfectly known at all time. 
\begin{figure}[!h]
	\centering
	\subfigure{\includegraphics[width=1.68in,height=1.6in]{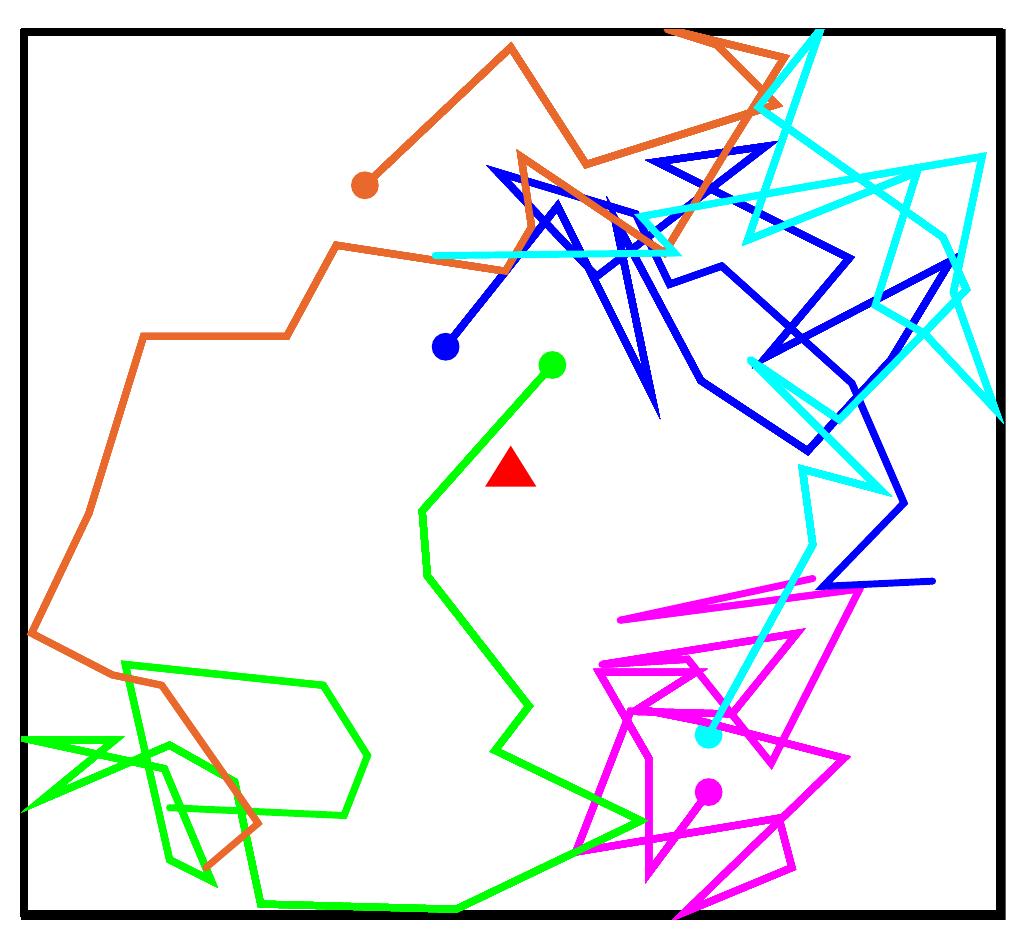}}
	\hspace{0.1cm}
	\subfigure{\includegraphics[width=1.68in,height=1.6in]{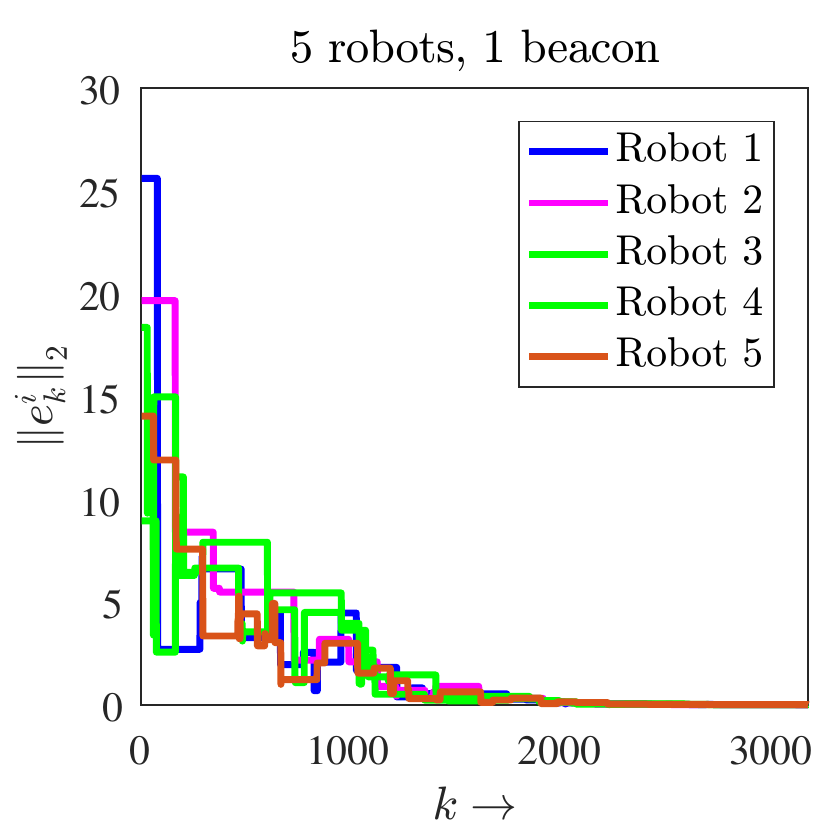}}
	\caption{(Left) Motion model;
		(Right) Convergence. 
		}
	\label{f1-2}
\end{figure}
In the beginning, all robots are randomly deployed inside a $20\mbox{m} \times 20\mbox{m}$ square. 

We fix the location of the beacon in the center of the region. For all simulations in this section, we set the communication radius to $r=2$m \textcolor{black}{and $d_{\max}=5\mbox{m}$}. All robots are initially assigned with a random estimate of their initial locations. If a robot finds at least~$m+1=3$ neighbors, it performs the inclusion test, as described in Section~\ref{ss:DILOC}. If the opportunity occurs, i.e., if a robot finds a triangulation set among the neighboring nodes, it updates its location estimate according to Eq.~\eqref{20}. The robot does not perform any update otherwise. \textcolor{black}{Throughout this section, we consider the scenario, where a robot performs multiple updates at each iteration with respect to all possible triangulation sets found in that iteration.}
 To ensure that the updating robot retains the valuable information it may have received from the beacon, and to guarantee a minimum contribution by the beacon when it is involved in an update, \textcolor{black}{we set~$\alpha_k=\beta=0.01$ and~$\alpha=0.01$, respectively.}

Fig.~\ref{f1-2} (Left) shows the motion model that we choose as random according to Eq.~\eqref{1}, i.e., the trajectories of $N=5$ mobile robots for the first~$20$ iterations. We choose the second norm of the error vector,~$\mb{e}_k^i$, to characterize the convergence, i.e.,
\textcolor{black}{\begin{equation}
{\Vert{\mb{e}^i_k}\Vert}_2=\sqrt{{{\left({{x}^i_k-{x}_k^{i\ast}}\right)}^2+{\left({{y}_k^i-{y}_k^{i\ast}}\right)}^2}}.
\end{equation}}
The algorithm converges to the true robot locations as~${\| \mb{e}_k^i \|_{2}} \rightarrow 0,\forall~i$.} The convergence of the algorithm in this case is illustrated in Fig.~\ref{f1-2} (Right) for one simulation.
\textcolor{black}{In Fig.~\ref{f1-3} (Left), we provide the convergence results for the networks with one beacon and $5$, $10$, $20$ and $100$ robots with unknown locations. Each curve indicates the average over $n=20$ Monte Carlo simulations.
\begin{figure}[!h]
	\centering
	\subfigure{\includegraphics[width=1.65in,height=1.65in]{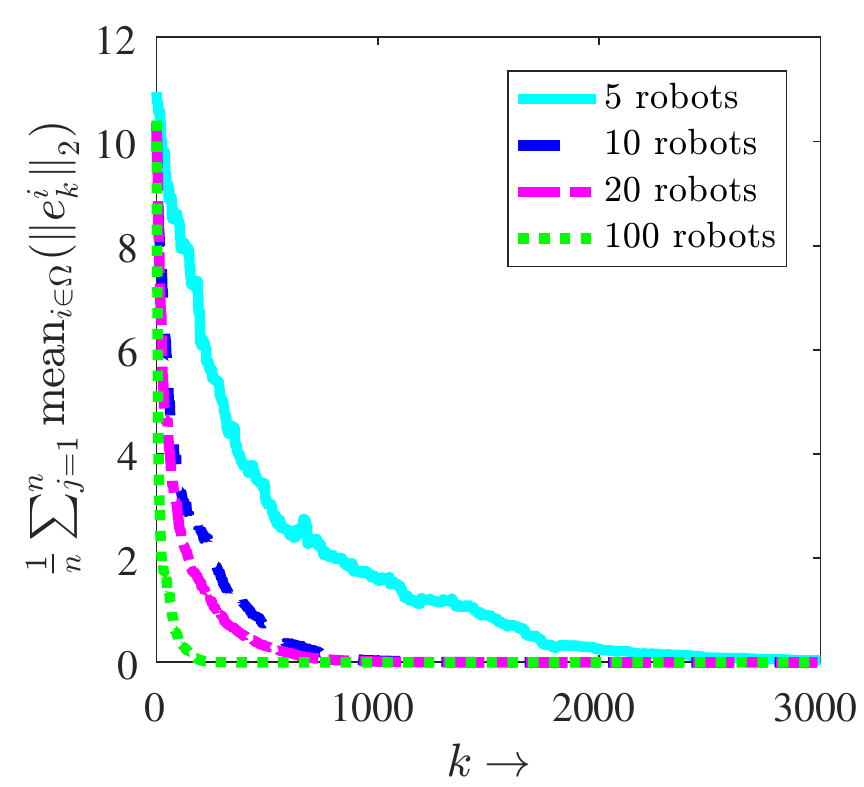}}
	\subfigure{\includegraphics[width=1.6in,height=1.6in]{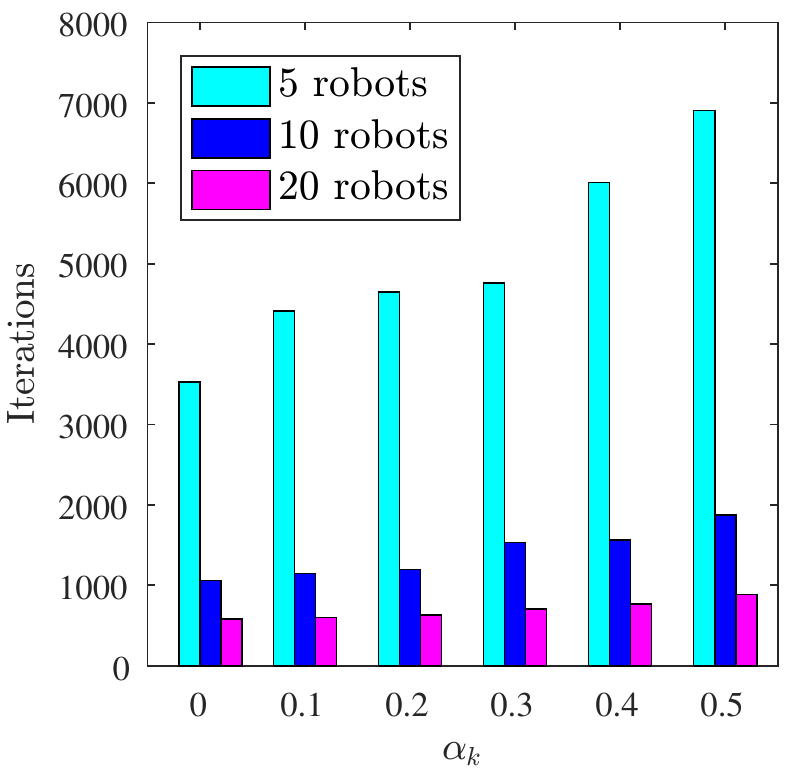}}
	\caption{\textcolor{black}{(Left) Convergence of networks with one beacon and $5$,~$10$,~$20$ and $100$ robots;
		(Right) Effect of self-weights on the convergence rate.}}
	\label{f1-3}
\end{figure}
Due to the opportunistic nature of the algorithm, as the number of robots increases each robot is more likely to find 
triangulation sets and perform successful updates. Therefore, the algorithm converges faster as the number of robots increases.}
\textcolor{black}{In Fig.~\ref{f1-3} (Right), we study the effect of the self-weights on the convergence rate of the algorithm in networks with one beacon and $5$, $10$ and $20$ robots as we change $\alpha_k=\beta$ from $0.01$ to $0.5$. It can be seen that the convergence rate decreases as $\alpha_k$ increases.}
\begin{figure}[!h]
	\centering
	\subfigure{\includegraphics[width=1.6in,height=1.6in]{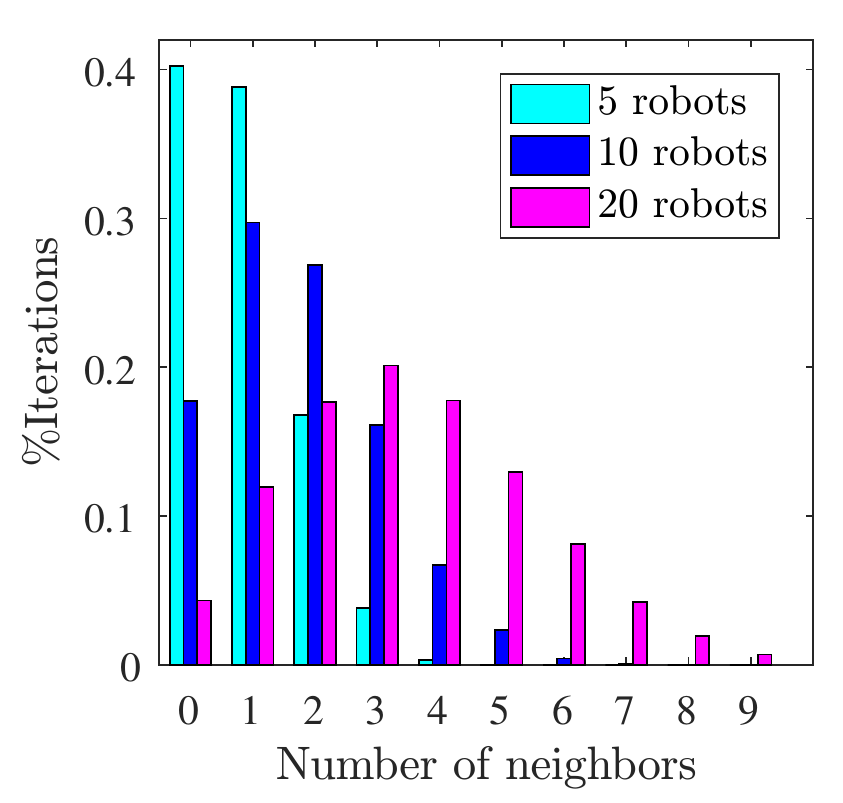}}
	\subfigure{\includegraphics[width=1.6in,height=1.6in]{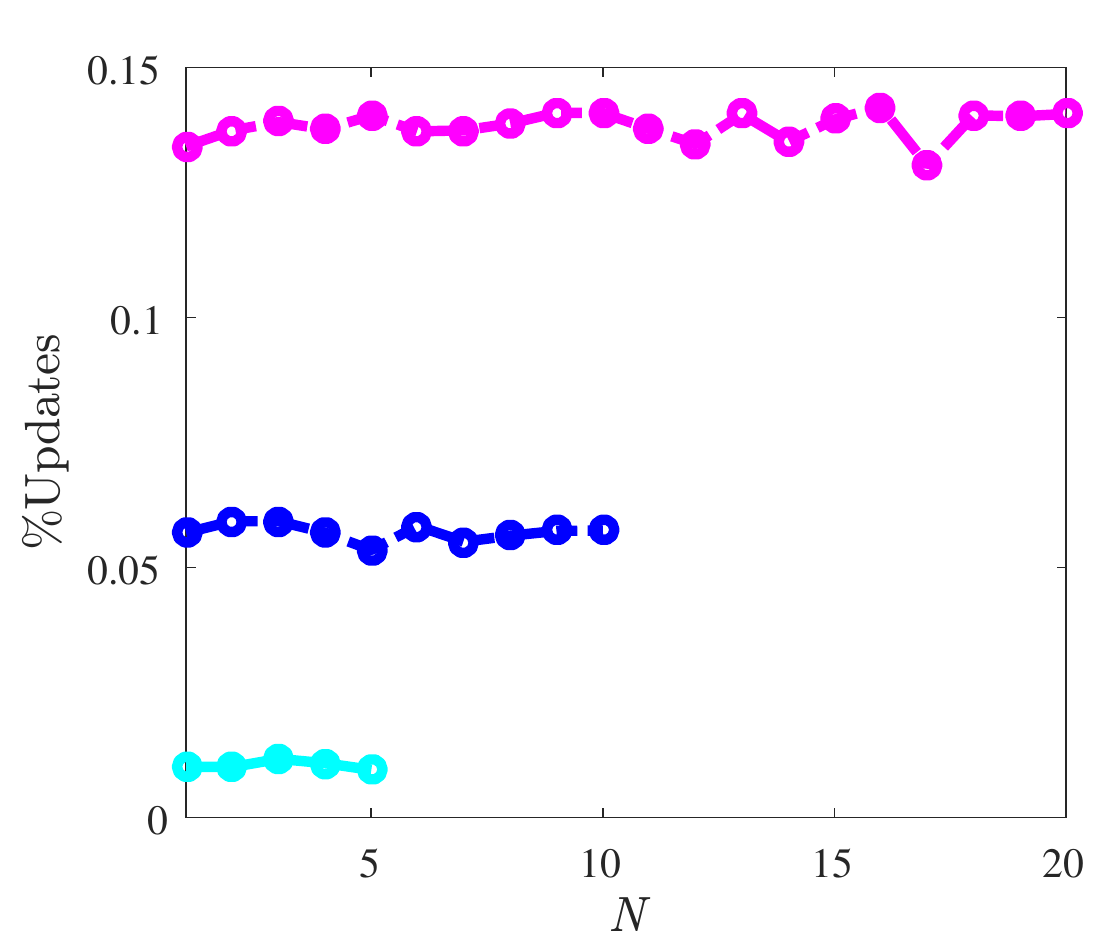}}
	\caption{\textcolor{black}{(Left) Number of neighbors;
		(Right) Number of updates/iterations.}}
	\label{f1-4}
\end{figure}

\textcolor{black}{In Fig.~\ref{f1-4} (Left), we show the percentage of the iterations a robot finds different number of neighbors in networks with one beacon and $5$, $10$ and $20$ robots. For example, a robot finds no neighbors during $40\%$ of the iterations in a network of $5$ robots. Fig.~\ref{f1-4} (Right) shows the ratio of the total number of updates to the number of iterations. We fix the number of iterations to $3000$, and take the average over $n=20$ Monte Carlo simulations. On average, a robot in networks with one beacon and $5$, $10$, and $20$ robots, performs~$31$, $171$, and $422$ updates, respectively.}

\textcolor{black}{We examine the effect of highly adverse initial conditions in Fig.~\ref{f1-5} (Left) for a network with one beacon and $100$ robots.
	 Despite very large initial errors, which is $30$ times larger than the dimension of the region, the algorithm converges in less than $500$ iterations in this simulation. This convergence rate is slower than the average convergence rate illustrated in Fig.~\ref{f1-3} (Left) for networks with $100$ robots primarily because of the large initial error.}
\begin{figure}[!h]
	\centering
	\subfigure{\includegraphics[width=1.68in,height=1.68in]{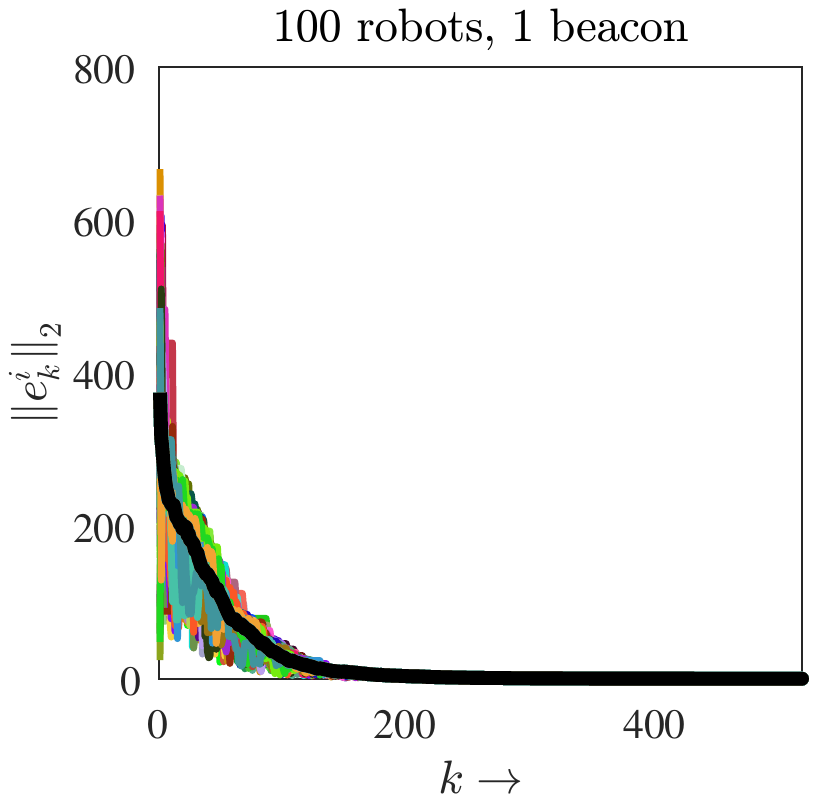}}
	\subfigure{\includegraphics[width=1.68in,height=1.68in]{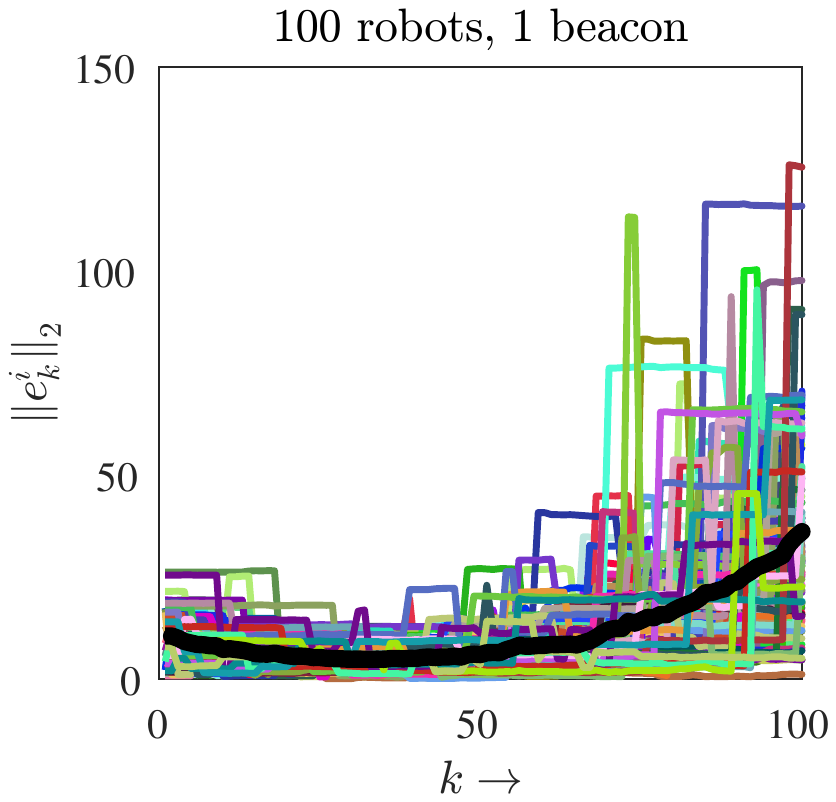}}
	\caption{\textcolor{black}{(Left) Convergence under highly adverse initial estimates;
		(Right) Effect of noise on the convergence; Thick black curves represent the mean error over all robots.}}
	\label{f1-5}
\end{figure}

\vspace{-5mm}

\subsection{Localization in the presence of noise}
\textcolor{black}{We use two different models to examine the effects of noise on the proposed localization algorithm; 
\begin{figure}[!h]
	\centering
	\subfigure{\includegraphics[width=1.5in,height=1.5in]{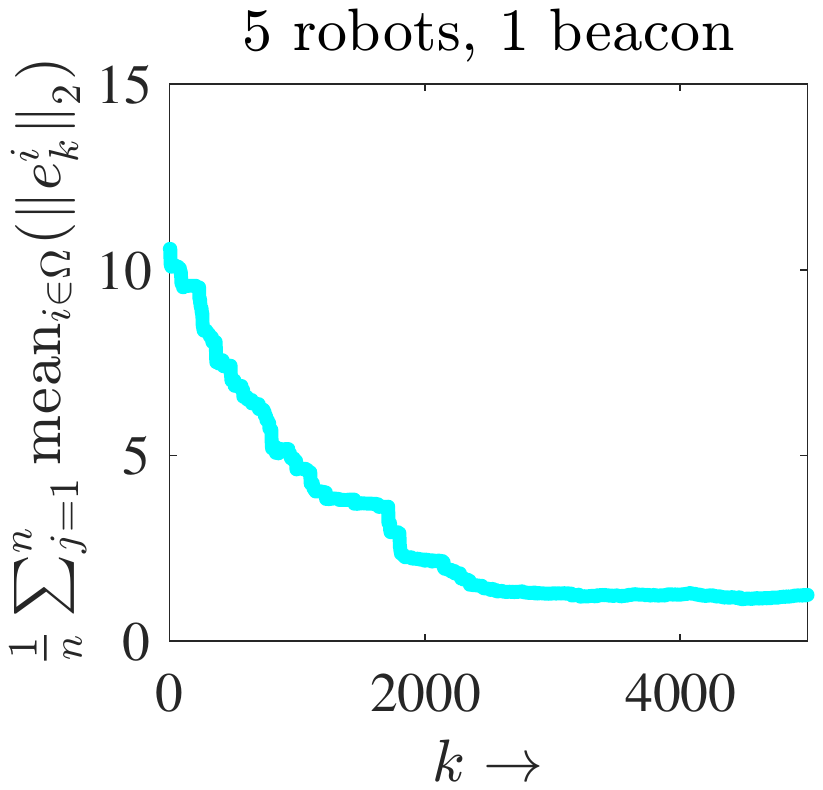}}
	\subfigure{\includegraphics[width=1.5in,height=1.5in]{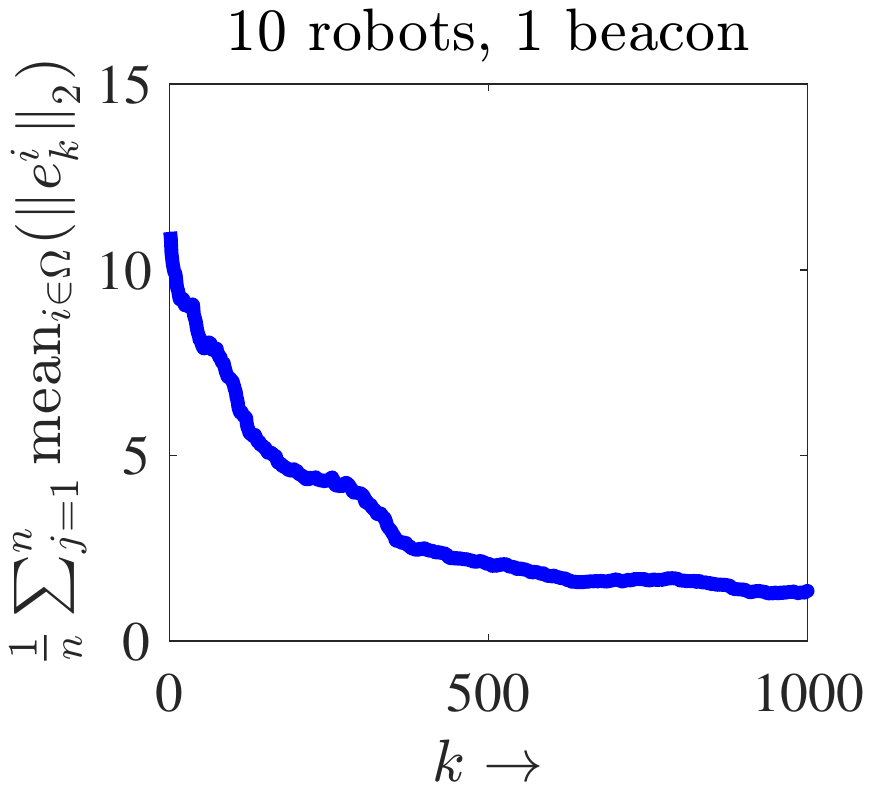}}
	\subfigure{\includegraphics[width=1.5in,height=1.5in]{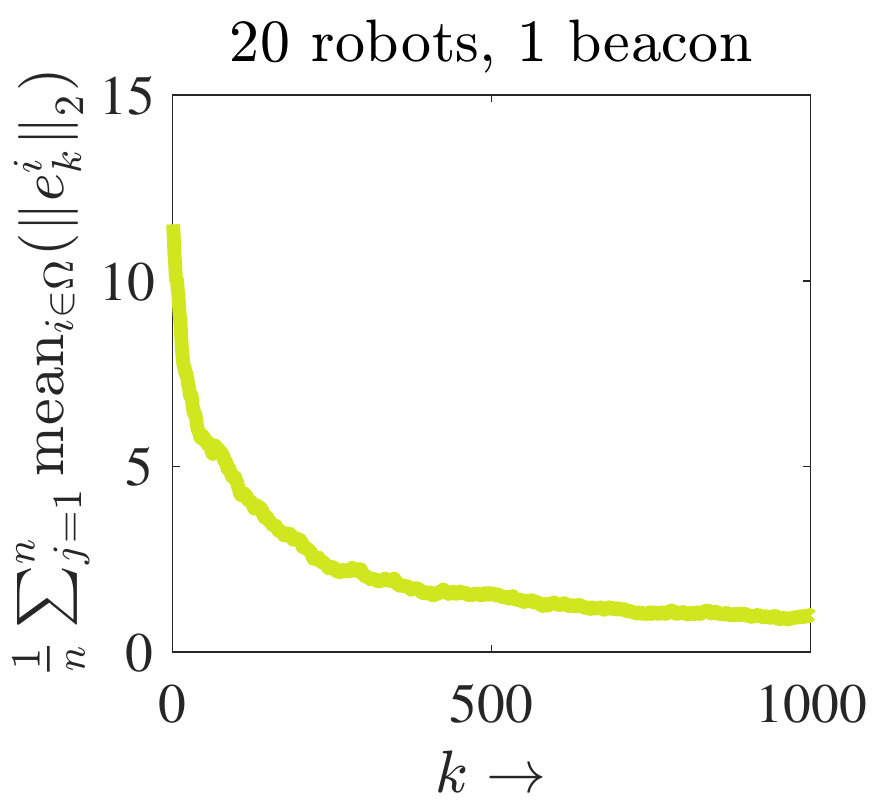}}
	\subfigure{\includegraphics[width=1.5in,height=1.5in]{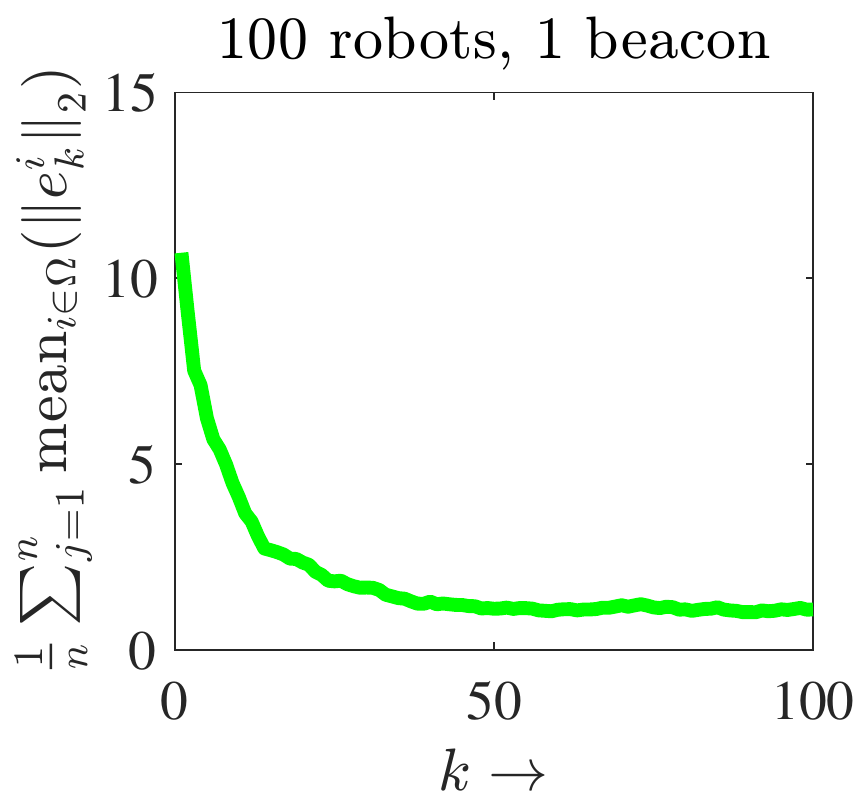}}
	\caption{\textcolor{black}{Modified algorithm under first noise model.}}
	\label{fig13}
\end{figure}	
	First, we assume that the noise on odometry measurements, i.e., the distance and angle that robot $i$ travels at time $k$, are Gaussian with zero mean and the following variances
\begin{eqnarray*}
		{\sigma_d^i}^2={K_d}^2D^i_k,~~
		{\sigma_{\theta}^i}^2={{K_{\theta}}^2} D^i_k,
\end{eqnarray*}
	where $D^i_k$ represents the total distance that robot $i$ has traveled up to time $k$.
	We also assume that the noise on the distance measurement (to a neighboring robot) at time $k$ is normal with zero mean and the variance~of~${\sigma_r^i}^2={K_r}^2 k$.
	Therefore the variances of the odometry measurements are proportional to the total distance a robot has traveled, and the variance on the distance measurements (to the  neighboring robots) increases with time.
	Such assumptions are common in the relevant literature, e.g.,~\cite{martinelli2005observability,chong1997accurate,wanasinghe2014decentralized,paull2014decentralized}.
\begin{figure*}
	\centering
	\subfigure{\includegraphics[width=1.75in,height=1.75in]{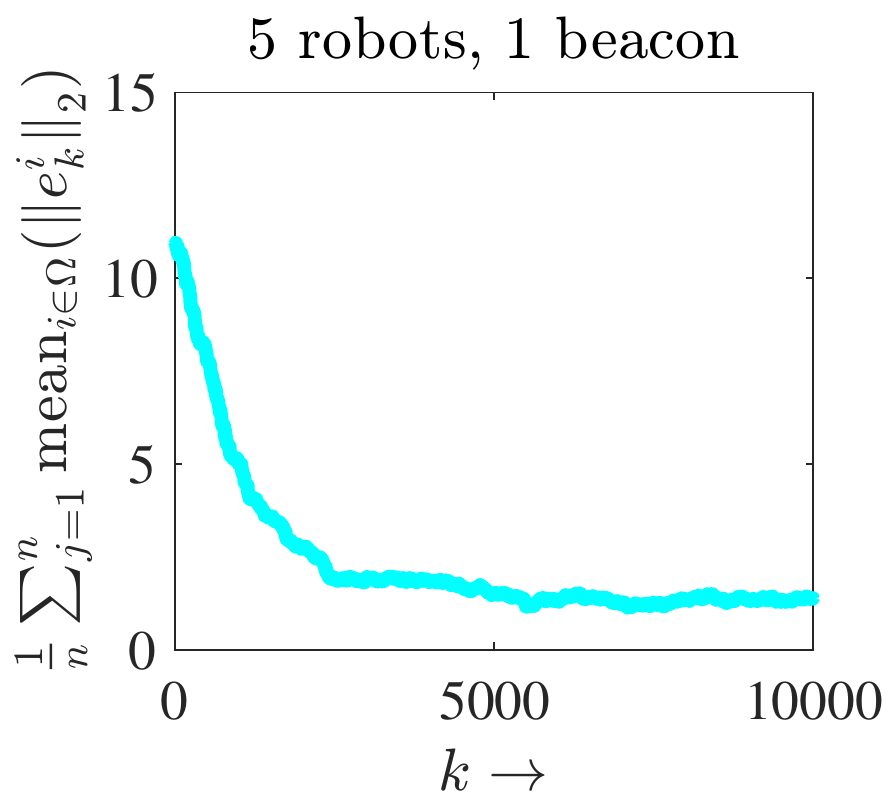}}
	\subfigure{\includegraphics[width=1.75in,height=1.75in]{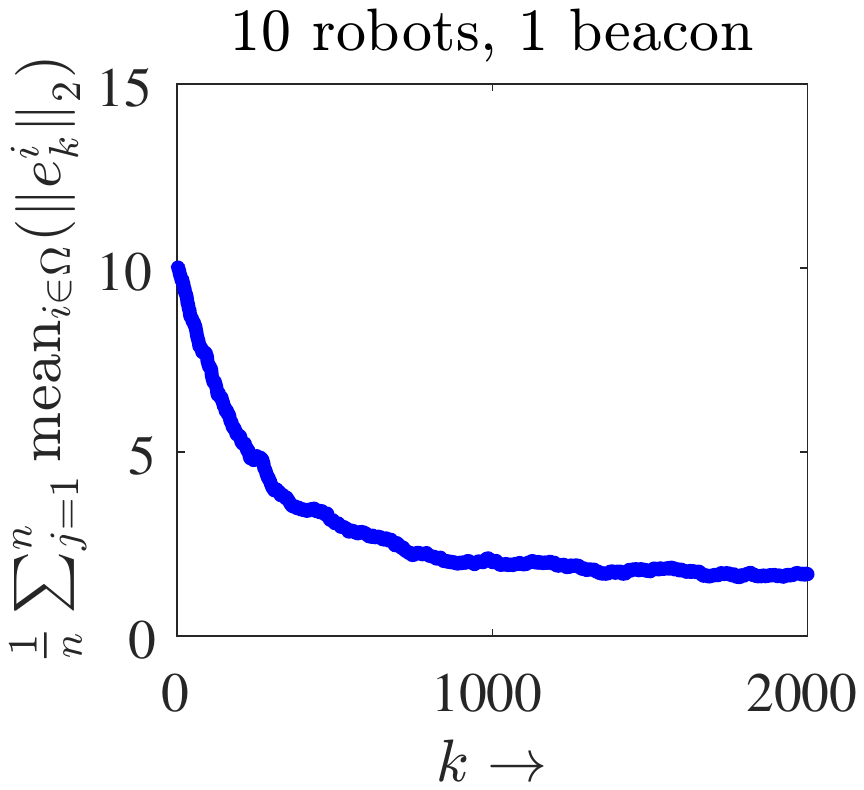}}
	\subfigure{\includegraphics[width=1.75in,height=1.75in]{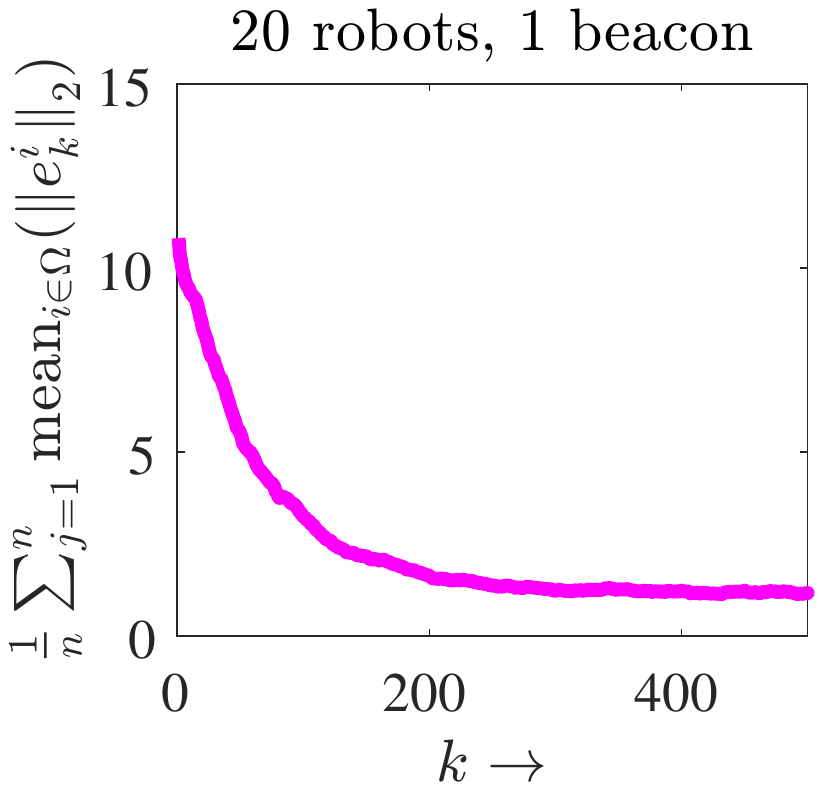}}
	\subfigure{\includegraphics[width=1.75in,height=1.75in]{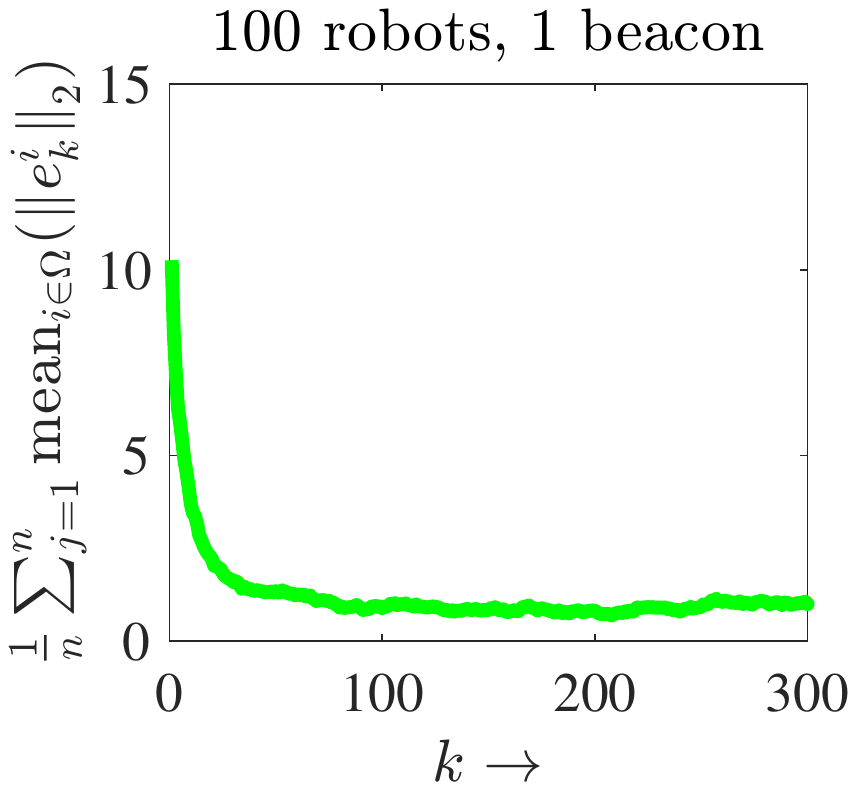}}
	\caption{\textcolor{black}{Modified algorithm under second noise model.}}
	\label{fig14}
\end{figure*}
As shown in Fig.~\ref{f1-5} (Left) for a network with one beacon and $100$ robots, setting ${K_d}=K_{\theta}=K_r=5*10^{-3}$ leads to an unbounded error, which is due to incorrect inclusion test results and the continuous location drifts because of the noise on the distance measurements and the noise on motion, respectively. 
However, by modifying the algorithm according to Section~\ref{noise}, it can be seen in Fig.~\ref{fig13} that the localization error is bounded by the communication radius. In the simulations with noise we choose $\epsilon=20\%$, i.e., a robot performs an update only if the relative inclusion test error, corresponding to the candidate triangulation set is less than $20\%$. }

\textcolor{black}{We evaluate the performance of the algorithm on a different noise model, where at each and every iteration the amount of noise on odometry and distance measurements are proportional to the measurements. In Fig.~\ref{fig14}, we show the simulation results when the amount of noise at each iteration is up to $\pm5\%$ of the measurements.
All the simulations in the presence of noise are averaged over $n=20$ Monte Carlo simulations.}

\subsection{\textcolor{black}{Performance evaluation}}\label{eval}
\textcolor{black}{We now evaluate the performance of our algorithm in contrast with some well-known localization methods; MCL~\cite{hu2004localization}, MSL*~\cite{rudafshani2007localization}, and Range-based SMCL~\cite{dil2006range}.
	\begin{figure}[!h]
		\centering
		\subfigure{\includegraphics[height=2in,width=2in]{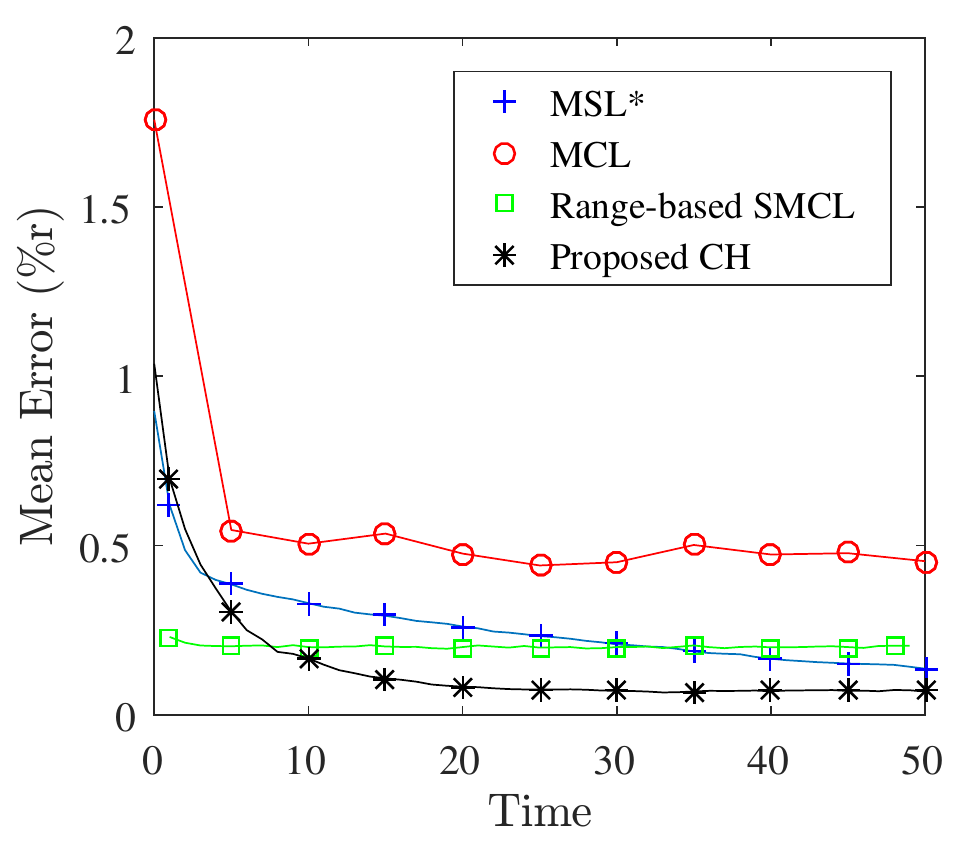}}
		\caption{\textcolor{black}{Accuracy comparison}}
		\label{f16}
	\end{figure} 
	Please refer back to Section~\ref{sec1} for a brief description of these methods.
	In Fig.~\ref{f16}, we compare the localization error in the Convex Hull (CH) algorithm with MCL, MSL*, and Range-based SMC. 
	As shown in Table~\ref{tbl1}, in these algorithms node density,~$n_d$, and beacon (seed) density,~$n_s$, denote as the average number of nodes and beacons in the neighborhood of an agent, respectively. Total number of agents and beacons can therefore be determined by knowing these densities as well as the area of the region. 
	We consider $N=100$ robots and $M=10$ beacons to remain consistent with the setup in~\cite{hu2004localization,rudafshani2007localization,dil2006range}, and use the same metric, i.e., the location error as a percentage of the communication range. Each data point in Fig.~\ref{f16} is computed by averaging the results of $20$ simulation experiments. We keep the other parameters the same as described earlier in this section. With high measurement noise, i.e., in the presence of $10\%$ noise on the range measurements and $1\%$ noise on the motion, our algorithm outperforms MCL, MSL* and Range-based SMCL after $10$ iterations. Clearly, the localization error in our algorithm decreases as the amount of noise decreases, and our algorithm converges to the exact robot locations in the absence of noise.}
\textcolor{black}{Table~\ref{tbl1} summarizes the performance of the proposed CH algorithm in comparison to the above methods.}
\begin{table}[!h]
	\centering
	\subfigure{\includegraphics[height=1.5in]{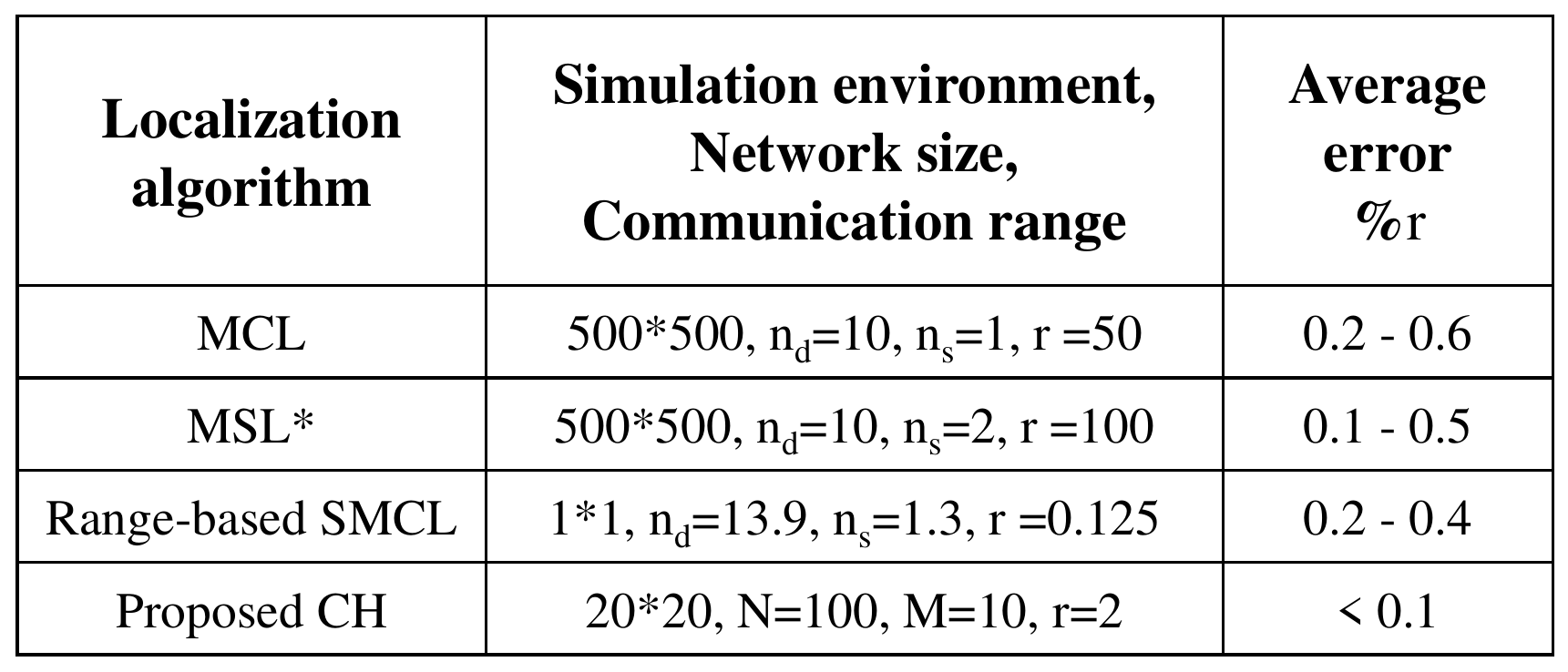}}
	\caption{\textcolor{black}{Comparative performance of localization algorithms}}
	\label{tbl1}
\end{table}

\vspace{-7mm}\section{Remarks}\label{chal}
\vspace{-3mm}\subsection{Convergence rate}
In some particular applications, e.g., search and rescue in hazardous environments, the robots need not only to find their locations, but also to finish the localization process successfully \textit{in a finite time. The convergence rate of the algorithm becomes crucial in such applications.} \textcolor{black}{Since the algorithm is asymptotic, one can design appropriate termination criteria that are application-dependent. For example, one such criterion can be designed according to the number of iterations typically needed given the size, mobility, models, and noise parameters, as evident from the simulation figures in Section~\ref{sec7}.
We do note that the convergence improves dramatically as the size of the network increases.}

\textcolor{black}{\vspace{-7mm}\subsection{Computational complexity}}\label{chalb}

\textcolor{black}{As explained in Appendx A, the Cayley-Menger determinant is the determinant of an $(m+2)\times(m+2)$ symmetric matrix that relates
	the distances among the $m+1$ points in a set, $\Theta_\ell \in \mathbb{R}^m$, to
	the volume of their convex hull. Thus, the dimension of such determinants only depends on the dimension of the corresponding Euclidean space. Considering our localization algorithm in $\mathbb{R}^2$, when a robot finds enough (at least $m+1=3$) neighbors, it has to calculate~$m+2=4$ Caley-Menger determinants of $4\times4$ matrices to perform the inclusion test. Similarly, in~$\mathbb{R}^3$, when a robot finds at least~$m+1=4$ neighbors, it needs to calculate $m+2=5$ determinants of~$5\times5$ matrices to perform the inclusion test. Note that to perform the inclusion test in~$\mathbb{R}^2$ and~$\mathbb{R}^3$, a robot has to compute four areas and five volumes, respectively. Since the complexity of the computation of an~$n\times n$ determinant is $\mathcal{O}(n!)$, considering at most one update per iteration, the computation complexity of the algorithm for an updating robot in $\mathbb{R}^m$ is $(m+2)\mathcal{O}((m+2)!)$, where $m\leq3$.}

\textcolor{black}{\vspace{-0mm}\subsection{Communication loss and drops}
In the event of a temporary communication loss/drop, a robot will not be able to perform an update and/or serve as a neighboring node in any triangulation set. However, since the algorithm is opportunistic and fully distributed, this is not going to affect the performance of the rest of the team except that it reduces the chance for the other robots to find a triangulation set. As the size of the network grows, such effect becomes negligible.
However, if the failure is permanent, i.e., one robot gets isolated from the network by the end of the localizing process, it will not be able to find its location and the size of the network gets reduced by one. 
}
\vspace{-0mm}\subsection{Challenges and future work}
We note that a successful implementation of this algorithm requires a knowledge of the associated parameters,~$\alpha,\alpha_k,\epsilon$. We provide some insight into these questions in the simulation section of this paper and have discussed the theoretical relevance of each one throughout the paper. In future, we will study how the changes in these parameters effect the convergence or performance of the underlying localization algorithm, which may require a large number of practical experiments. As a specific example, if the statistics of the distance/motion noise variables are known, one may be able to design~$\epsilon$ so the inclusion test is correct (with a very high probability) regardless of the noise parameters. \textcolor{black}{As part of our future work, we are also planning to test our localization algorithm in our lab both on
UGV’s (in R2) and UAV’s (in R3).}

\section{Conclusions}\label{sec8}
In this paper, we provide a \emph{linear} distributed algorithm to localize an arbitrary number of mobile robots moving in a bounded region. We assume that each robot can measure a noisy version of its motion as well as its distance to the neighboring nodes. We consider an opportunistic algorithm such that the robots update their location estimates in~${\mathbb{R}}^m$, as a linear-convex combination of their~$m+1$ neighbors if they lie inside their convex hull. The updating robot uses the opportunistic information exchange to refine its location estimate by using a barycentric-based convex update rule. We abstract the algorithm as an LTV system with (sub-)stochastic matrices, and show that it converges to the true robot locations under some mild regularity conditions on update weights. We also relate the dimension of motion in the network to the number of beacons required, and 
 show that a network of mobile robots with full degrees of freedom in the motion can be localized precisely, as long as there is at least one 
 beacon in the network. 
We evaluate the performance of the algorithm \textcolor{black}{in the presence of noise and provide modifications to the proposed algorithm 
	to address the undesirable effects of noise.}
\appendices
\textcolor{black}{\vspace{-8mm}\section{Cayley-Menger determinant}\label{cm}
Consider a set of $m+1$ points, $\Theta_\ell$, in $m$-dimensional Euclidean space. Cayley-Menger (CM) determinant is the determinant
of an~$(m+2)\times(m+2)$ symmetric matrix, which uses the pairwise distance information in the set $\Theta_\ell$, to compute the hypervolume,~$A_{\Theta_\ell}$, of their convex hull, $\mathcal{C}(\Theta_\ell)$. The CM determinant is given by
\begin{equation}\label{cmeq}
A_{\Theta_\ell}^2=\frac{1}{s_{m+1}}
\begin{vmatrix}
0 & {\bf{1}}_{m+1}^T\\
{\bf{1}}_{m+1} & {\bf{D}}
\end{vmatrix},
\end{equation}
in which~${\bf{1}}_{m+1}$ denotes an $m+1$-dimensional column
vector of $1$s,~${\bf{D}}=\{d_{{\ell}j}^2\},~\ell,j\in\Theta_l$, is the~$(m+1)\times(m+1)$ matrix of squared distances, $d_{{\ell}j}$, within the set, $\Theta_\ell$, and
\begin{equation}\label{cmeq2}
s_m=\frac{2^m{(m!)}^2}{{(-1)}^{m+1}},~m\in\{0,1,2,\ldots \}.
\end{equation}
The second and third coefficients in the above sequence that are relevant in $\mathbb{R}^2$ and $\mathbb{R}^3$ are $-16$ and $288$, respectively.}

\bibliographystyle{IEEEtran}
\bibliography{bibliography_Arxiv}
\end{document}